\pdfoutput=1

\documentclass{article}

\usepackage{microtype}
\usepackage{graphicx}
\usepackage{subfigure}
\usepackage{booktabs} 

\usepackage{hyperref}

\usepackage[accepted]{icml2024}

\usepackage{amsmath}
\usepackage{amssymb}
\usepackage{mathtools}
\usepackage{amsthm}
\usepackage{enumitem}
\usepackage{mdframed}

\usepackage[capitalize,noabbrev]{cleveref}

\theoremstyle{plain}
\newtheorem{theorem}{Theorem}[section]
\newtheorem{proposition}[theorem]{Proposition}
\newtheorem{lemma}[theorem]{Lemma}

\theoremstyle{definition}
\newtheorem{definition}[theorem]{Definition}
\newtheorem{assumption}[theorem]{Assumption}
\theoremstyle{remark}
\newtheorem{remark}[theorem]{Remark}

\usepackage[textsize=tiny]{todonotes}

\icmltitlerunning{An Improved Finite-time Analysis of Temporal Difference Learning with Deep Neural Networks}

\begin{document}

\twocolumn[
\icmltitle{An Improved Finite-time Analysis of Temporal Difference \\ 
   Learning with Deep Neural Networks}




\begin{icmlauthorlist}
\icmlauthor{Zhifa Ke}{zfke1}
\icmlauthor{Zaiwen Wen}{wen1}
\icmlauthor{Junyu Zhang}{zjy1}
\end{icmlauthorlist}

\icmlaffiliation{zfke1}{Center for Data Science, Peking University, China.}
\icmlaffiliation{zjy1}{Department of Industrial Systems Engineering and Management, National University of Singapore, Singapore}
\icmlaffiliation{wen1}{Beijing International Center for Mathematical Research, Center for Machine Learning Research and Changsha Institute for Computing and Digital Economy, Beijing, China}

\icmlcorrespondingauthor{Zhifa Ke}{kezhifa@stu.pku.edu.cn}

\icmlcorrespondingauthor{Zaiwen Wen}{wenzw@pku.edu.cn}

\icmlcorrespondingauthor{Junyu Zhang}{junyuz@nus.edu.sg}

\icmlkeywords{Machine Learning, ICML}

\vskip 0.3in
]



\printAffiliationsAndNotice{}  

\begin{abstract}
Temporal difference (TD) learning algorithms with neural network function parameterization have well-established empirical success in many practical large-scale reinforcement learning tasks. However, theoretical understanding of these algorithms remains challenging due to the nonlinearity of the action-value approximation. In this paper, we develop an improved non-asymptotic analysis of the neural TD method with a general $L$-layer neural network. New proof techniques are developed and an improved new $\tilde{\mathcal{O}}(\epsilon^{-1})$ sample complexity is derived. To our best knowledge, this is the first finite-time analysis of neural TD that achieves an $\tilde{\mathcal{O}}(\epsilon^{-1})$ complexity under the Markovian sampling, as opposed to the best known $\tilde{\mathcal{O}}(\epsilon^{-2})$ complexity in the existing literature.
\end{abstract}

\section{Introduction}
\setlength{\parindent}{0pt}
The temporal difference (TD) learning method, firstly designed for policy evaluation \cite{sutton1988learning}, is a fundamental building block of many popular Reinforcement Learning (RL) algorithms. In standard TD learning algorithms for tabular MDP, based on the Bellman operator, the agent iteratively obtains a state-action-reward-transition tuple and then updates the Q values by a weighted average of the current value and the TD target. Once the algorithm converges, the Q function is considered to be the final return obtained by executing the target policy given some initial action-state pair.

For large-scale reinforcement learning (RL) problems,  appropriate parameterization of the Q function is crucial for better scalability of the TD algorithms. Common examples include linear \cite{tesauro1995temporal}, general smooth nonlinear \cite{maei2009convergent}, and neural network \cite{mnih2013playing} function approximations. However, it is well known that the naive extension of TD learning and Q-learning algorithms can diverge under the general function approximation \citet{tsitsklis}. To encourage convergence, numerous variants of TD and Q-learning have been proposed, including Least-squares TD (LSTD) \cite{bradtke1996linear, boyan2002technical} and gradient TD (GTD) \cite{sutton2009fast,sutton2009convergent}, to name a few.

The applications of neural network function approximation have witnessed huge empirical success in many real-world tasks, including Deep Q-network (DQN) algorithms \cite{mnih2013playing, van2016deep}, policy improvement method \cite{sutton1999policy}, trust region policy optimization \cite{schulman2015trust} and the actor-critic algorithms \cite{konda1999actor,lillicrap2015continuous,fujimoto2018addressing}, etc. However, due to the analysis difficulties brought by the function approximation, a significant gap exists between the empirical success and the theoretical understanding of these algorithms. Hence analyzing the convergence and sample complexity of TD learning and Q-learning under various Q function parameterizations has always been an active topic in the RL community during the past decades.

Early works focus on the asymptotic convergence of the algorithms with tabular or linear function approximation. For the tabular (stochastic) TD or Q learning method,  \citet{jaakkola1993convergence} established the asymptotic convergence for the first time. Later on, the asymptotic convergence of algorithms with linear function approximation has been extensively discussed using ODE-based methods, see e.g. \citet{tsitsklis, perkins2002existence, borkar2009sto}. 
Meanwhile, in contrast to the convergent results for RL algorithms under the tabular or linear settings, TD with nonlinear function approximation is known to diverge in general \citet{tsitsklis,  brandfonbrener2019geo}. To overcome this issue, \citet{maei2009convergent} proposed to optimize the Mean Squared Projected Bellman Error (MSPBE) via a gradient-based algorithm. Due to the problem nonconvexity,
only asymptotic convergence to stationary points can be guaranteed. 

More recently, benefiting from the improved techniques for analyzing stochastic optimization algorithms, there has been a growing number of research on providing finite-time analysis for TD and Q-learning algorithms with function approximations. 

For linear function approximation, the non-asymptotic results of TD learning and its variants are relatively well-understood, including TD \citet{bhandari2018finite, dalal2018finite, zou2019finite}, gradient TD \citet{dalal2018finite, touati2018convergent, liu2020finite}, and Least-Squares TD \citet{lazaric2010finite, prashanth2014fast, tagorti2015rate}, etc. In particular, \citet{bhandari2018finite}  established the first finite-time analysis of linear Q-learning under both i.i.d. sampling and Markovian sampling settings.

For neural network function approximation, which is directly related to this paper, we provide a more detailed discussion. Based on the recent advances in the understanding of optimizing ReLU network  \citet{jacot2018neural, du2018gradient, allen2019learning, allen2019convergence, cao2019generalization, cao2020generalization}, a few recent works have successfully developed the finite-time analysis of the neural TD and neural Q-learning algorithms, as long as the Q network is sufficiently wide. Let $Q^*$ be the true action-value function and let $Q(s,a;\boldsymbol{\theta})$ denote the action-value function parameterized by a neural network with weights $\boldsymbol{\theta}$, at any state action pair $(s,a)$. Then we aim to find some $\epsilon$-optimal parameter $\Bar{\boldsymbol{\theta}}$ such that $\mathbb{E}\big[(Q(s,a;\Bar{\boldsymbol{\theta}})-Q^*(s,a))^2\big]\leq\epsilon+\epsilon_\mathcal{F}$, where the expectation is taken over the possible randomness in the output $\bar{\theta}$ as well as the distribution over the state-action pairs $(s,a)$, and $\epsilon_\mathcal{F}$ is the optimal approximation error of the parameterization function class. In \cite{xu2020finite}, a neural Q-learning algorithm with a general $L$-layer ReLU network is analyzed, and an $\tilde{\mathcal{O}}(\epsilon^{-2})$
sample complexity is guaranteed given that the network is sufficiently wide. In \cite{cai2023neural}, the authors studied both the neural TD learning and neural Q-learning algorithms for minimizing the MSPBE for policy evaluation and policy optimization, respectively. For policy evaluation, the $Q^*$ in the definition of an $\epsilon$-optimal solution is defaulted as $Q^\pi$ with $\pi$ being the policy to be evaluated. For both cases, an $\tilde{\mathcal{O}}(\epsilon^{-2})$ sample complexity is guaranteed for wide two-layer ReLU networks. 
In \cite{sun2022finite}, an $\tilde{O}(\epsilon^{-\frac{2}{2-\alpha}})$ complexity has been achieved by an adaptive neural TD algorithm with multi-layer ReLU networks, where $\alpha\in(0,1]$ is a constant that characterizes the sparsity and decay rate of the stochastic semi-gradients. However, without additional assumption, only an $\tilde{O}(\epsilon^{-2})$ complexity with $\alpha=1$ can be theoretically guaranteed. Finally, for policy evaluation problems, there are also several works that aim at reducing the width of the over-parameterized Q networks in the existing works \cite{tian2022performance,cayci2023sample}. In terms of complexity, both of them requires $\tilde{\mathcal{O}}(\epsilon^{-2})$ samples to obtain an $\epsilon$-optimal solution.

Despite the fact that existing analysis of the neural TD or neural Q-learning algorithms merely provides the $\tilde{\mathcal{O}}(\epsilon^{-2})$ sample complexity under various settings, an $\tilde{\mathcal{O}}(\epsilon^{-1})$ sample complexity should be expected. In fact, a double-loop fitted Q-iteration (FQI) method \cite{fan2020theoretical} and its single-loop Gauss-Newton variant \cite{ke2023provably} can achieve an $\tilde{\mathcal{O}}(\epsilon^{-1})$ sample complexity is obtained for two-layer Q networks.  Let $\mathcal{T}$ be the Bellman (optimality) operator, then the FQI method repeatedly solves a nonlinear least square subproblem to obtain the next iteration: $\boldsymbol{\theta}_{k+1}\approx\mathop{\mathrm{argmin}}_{\boldsymbol{\theta}\in\Theta} \mathbb{E}\big[(Q(s,a;\boldsymbol{\theta})-\mathcal{T}Q(s,a;\boldsymbol{\theta}_k))^2\big]$. Compared to the single-loop neural TD or neural Q-learning method that takes only one sample (or a mini-batch) to update the weights of Q networks, the update scheme of FQI requires repeatedly solving a subproblem to sufficiently high accuracy to enable convergence, which makes it inefficient and less favorable in practice. Therefore, we would like to raise a question: \vspace{0.2cm}
\begin{mdframed}[leftmargin=1cm,rightmargin=1cm, backgroundcolor=gray!10]
\emph{Can we improve the existing analysis of the neural 
temporal difference learning algorithm and obtain an $\tilde{\mathcal{O}}(\epsilon^{-1})$ sample complexity under general multi-layer Q neural networks?}
\end{mdframed}

To answer this question, we revisit the convergence analysis of the neural TD learning or Q-learning algorithms under the non-i.i.d. Markovian observations where a general $L$-layer neural network is used for Q function parameterization. By proposing a new subspace analysis technique, under suitable conditions, we derive a brand new $\tilde{\mathcal{O}}(\epsilon^{-1})$ sample complexity for neural TD learning or Q-learning, improving the state-of-the-art $\tilde{\mathcal{O}}(\epsilon^{-2})$ sample complexity in the existing works. Our contributions are summarized as follows.

\begin{table*}[t]
\vspace{-0.1cm} 
\begin{center}
\tabcolsep=0.1cm
\begin{small}
\begin{tabular}{lcccccr}
\toprule
\  & Neural Approximation & Network Depth & Network Width & Activation & Sample Complexity \\
\midrule
\cite{bhandari2018finite}   & No & NA & NA & NA & $\mathcal{O}(1/\epsilon)$ \\
\cite{cai2023neural}   & Yes & 2 & $\Omega(1/\epsilon^4)$ & ReLu & $\mathcal{O}(1/\epsilon^2)$ \\
\cite{xu2020finite}   & Yes & $L$ & $\Omega(1/\epsilon^6)$ & ReLu & $\mathcal{O}(1/\epsilon^2)$ \\
\cite{sun2022finite}  & Yes & $L$ & $\Omega(1/\epsilon^6)$ & ReLu & $\mathcal{O}(1/\epsilon^{\frac{2}{2-\alpha}}), \alpha\in(0,1]$ \\
\cite{tian2022performance}  & Yes & $L$ & $\Omega(1/\epsilon^2)$ & ELU, GeLU & $\mathcal{O}(1/\epsilon^2)$ \\
Ours  & Yes & $L$ & $\Omega(1/\epsilon^2)$ & ELU, GeLU & \textbf{$\mathcal{O}(1/\epsilon)$} \\
\bottomrule
\end{tabular}
\caption{Sample complexity for parameterized Q learning to find some $\bar{\boldsymbol{\theta}}$ such that $\mathbb{E}\left[\|Q(s,a;\bar{\boldsymbol{\theta}})-Q^*(s,a)\|_\mu^2\right]\leq \varepsilon$, where $\|f\|_\mu^2:=\int |f|^2 d \mu$ and $Q^*(s,a)$ satisfies the Bellman optimality equation $Q^*(s,a)=\mathcal{T} Q^*(s,a)$.}
\label{tab:complexity}
\end{small}
\end{center}
\end{table*}

\begin{itemize}[leftmargin=*]
\item Under the non-i.i.d. Markovian sampling setting, we derive an $\tilde{\mathcal{O}}(\epsilon^{-1})$ sample complexity for both neural TD learning and Q-learning methods under the multi-layer network approximation for Q functions. Our result also improves the best known $\tilde{\mathcal{O}}(\epsilon^{-2})$ sample complexity in the existing works. 
\item Based on our newly developed techniques, we further provide a finite-sample analysis for a minimax neural Q-learning algorithm that solves two-player zero-sum Markov games. An $\tilde{\mathcal{O}}(\epsilon^{-1})$ sample complexity is obtained under the non-i.i.d. Markovian sampling setting.
\end{itemize}

Technically, the subspace analysis approach that we propose to establish the $\tilde{O}(\epsilon^{-1})$ sample complexity is by itself of independent interest.  We believe this technique can potentially be applied to linear Q-learning algorithms and linear Actor-Critic algorithms without requiring the positive definiteness assumption of the feature covariance matrix \cite{bhandari2018finite, zou2019finite, barakat2022analysis}, while maintaining the $\tilde{O}(\epsilon^{-1})$ complexity.

In summary, we provide a comprehensive comparison between our work and the most related works in their respective settings and sample complexity in Table \ref{tab:complexity}. Our work establishes an optimal sample complexity analysis within a broader contextual framework.

\section{Preliminaries}
\label{sect:prelim}
We consider the infinite-horizon discounted Markov decision process (MDP), which is denoted as $\mathcal{M}=(\mathcal{S}, \mathcal{A}, \mathbb{P}, r, \gamma)$. We consider a general state space $\mathcal{S}$ and a finite action space $\mathcal{A}$. At any state $s\in\mathcal{S}$, if the agent takes an action $a\in\mathcal{A}$, it will receive a reward $r(s,a)\in[-R_{\max}, R_{\max}]$ and transition to the next state $s'\in\mathcal{S}$ with probability $\mathbb{P}(s'|s,a)$. We call $r$ the reward function and $\mathbb{P}$ the transition kernel. Let $\gamma\in(0,1)$ be a discount factor, then an MDP aims to find a sequence of actions $\{a_t\}_{t\geq0}$ to maximize the expected and discounted cumulative reward
$\mathbb{E}\big[\sum_{t=0}^{\infty} \gamma^{t} \!\cdot r(s_t,a_t) | s_{0}\sim\mu\big]$, where $\mu$ is the distribution of the initial state $s_0$.

Let $\Delta_{\mathcal{A}}$ denote the set of all probability distributions over the action space $\mathcal{A}$, and let a policy $\pi:\mathcal{S}\mapsto\Delta_{\mathcal{A}}$ be a mapping that returns a probability distribution $\pi(\cdot|s)\in\Delta_{\mathcal{A}}$ given any state $s\in\mathcal{S}$. If an agent follows a policy $\pi$, then at any state $s_t$, it will act by sampling an action $a_t\sim\pi(\cdot|s_t)$. Therefore, the action-value function (Q-function) under the policy $\pi$ is  
$$Q^{\pi}(s,a):=\mathbb{E}_\pi\left[\sum_{t=0}^{\infty} \gamma^{t} \cdot r(s_t,a_t) | s_{0}=s, a_{0}=a\right], $$
for $\forall (s,a)\in\mathcal{S}\times\mathcal{A}$,
where all actions except $a_0$ are sampled according to $\pi$.
For any mapping $Q: \mathcal{S}\times \mathcal{A} \rightarrow \mathbb{R}$, let the Bellman operator $\mathcal{T}^\pi$ be
$$
\begin{aligned}
  \mathcal{T}^\pi  Q(s,a):=r(s, a)+\gamma & \mathbb{E}\left[Q(s', a')\mid s'\sim \mathbb{P}(\cdot\mid s,a),\right. \\
  &\left.a'\sim \pi(\cdot\mid s')\right], \,\,\forall s,a.
\end{aligned}
$$
Then $\mathcal{T}^\pi$ is a $\gamma$-contraction under the infinity norm and $Q^\pi$ is the unique solution to the  fixed-point equation $Q=\mathcal{T}^\pi Q$ \cite{bertsekas2012dynamic}. If the Q function is parameterized by some function $Q(s,a;\boldsymbol{\theta})$ to gain better scalability for large-scale RL problems, popular approaches for finding a good $\boldsymbol{\theta}$ include minimizing the 
the Mean-Squared Bellman Error (MSBE):  
\begin{equation}
\label{eq:msbe}
\min_{\boldsymbol{\theta}\in\Theta} \mathbb{E}_{(s,a)\sim\mu} \Big[\left(Q(s,a;\boldsymbol{\theta})-\mathcal{T}^\pi Q(s,a;\boldsymbol{\theta})\right)^2\Big],
\end{equation}
and minimizing the Mean-Squared Projected Bellman Error (MSPBE):
\begin{equation}
\label{eq:mspbe}
\min_{\boldsymbol{\theta}\in\Theta}  \mathbb{E}_{(s,a)\sim\mu} \left[\left(Q(s,a;\boldsymbol{\theta})-\Pi_{\mathcal{F}} \mathcal{T}^\pi Q(s,a;\boldsymbol{\theta})\right)^2\right],
\end{equation}
where $\Theta$ is a feasible domain of the parameter $\boldsymbol{\theta}$, $\mu$ is some distribution over state action pairs, and $\Pi_\mathcal{F}$ is the projection onto some function class $\mathcal{F}$. Typical choices of $\mathcal{F}$ includes the Q function parameterization class itself $\mathcal{F}:=\{Q(\cdot;\boldsymbol{\theta}):\theta\in\Theta\}$ \cite{maei2009convergent}, and some local linearization of the parameterization function class \cite{cai2023neural}.

In this paper, we study the neural temporal difference 
learning method where the action-value function is parameterized by some multi-layer neural network. Let us define a feedforward neural network by the following recursion:
\begin{equation}
\label{eq:nn-layer}
\boldsymbol{x}^{(l)}=\frac{1}{\sqrt{m}} \sigma \left(\boldsymbol{W}_{l}\boldsymbol{x}^{(l-1)}\right), \quad l\in \{1,2,\cdots,L\},
\end{equation}
where $\boldsymbol{W}_1\in\mathbb{R}^{m\times d}$, $\boldsymbol{W}_{l}\in\mathbb{R}^{m\times m}$ for $2\leq l\leq L$ are the weight matrices of the network, $\sigma(\cdot)$ is  
an activation function, and the input is a feature map $\boldsymbol{x}^{(0)}=\phi(s,a)\in\mathbb{R}^d$ for any state action pair $(s,a)$. For simplicity of notation, we write $\boldsymbol{x}=\boldsymbol{x}^{(0)}$, then $Q(s,a)$ is parameterized by
\begin{equation}
\label{eq:nn}
Q(\boldsymbol{x};\boldsymbol{\theta})=\frac{1}{\sqrt{m}} \boldsymbol{b}^\top \boldsymbol{x}^{(L)} , 
\end{equation}
where the parameter $\boldsymbol{\theta}=\left(\mbox{Vec}(\boldsymbol{W}_1);\cdots; \mbox{Vec}(\boldsymbol{W}_L)\right)$ denotes the collection of all weight matrices, and $\boldsymbol{b}$ is given by a random initialization. $\mbox{Vec}(\cdot)$ stands for the vetorization operator that reshapes a matrix to a column vector by stacking its columns one by one and the ``;'' separator in $\boldsymbol{\theta}$ stands for the vertical stacking of the elements. That is, we reshape $\boldsymbol{\theta}$ to a long column vector for the notational convenience in later discussion. 

\begin{assumption}
\label{as:activation}
The activation function $\sigma(\cdot)$ is $L_1$-Lipschitz and $L_2$-smooth, i.e. , for $\forall y_1, y_2\in \mathbb{R}:$
$$
\left|\sigma(y_1)-\sigma(y_2)\right| \leq L_1 |y_1-y_2|
$$
and 
$$
\left|\sigma'(y_1)-\sigma'(y_2)\right| \leq L_2 |y_1-y_2|.
$$
\end{assumption}
Assumption \ref{as:activation} indicates that our results below are not based on the popular ReLU activation function. However, we primarily focus on some twice-differentiable activation functions (such as Sigmoid, ELU, GeLU, etc.), which are smooth approximations of the ReLU function and are frequently utilized in practical problems \cite{devlin2018bert, godfrey2019evaluation}. Such a setup aligns with \cite{liu2020linearity}, and provides a $\mathcal{O}(m^{-\frac{1}{2}})$-smooth property for the neural Q-function.

Let $\boldsymbol{\theta}^0=\left(\mbox{Vec}(\boldsymbol{W}_1^0);\cdots; \mbox{Vec}(\boldsymbol{W}_L^0)\right)$ be the initial solution. For each $l$, we initialize the weights of $\boldsymbol{W}_l^0$ 
element-wise from a normal distribution $\mathcal{N}(0,1)$ and each element of $\boldsymbol{b}$ is drawn uniformly from $\{-1,+1\}$. The parameter $\boldsymbol{b}$ will not be optimized during training. For regularity purpose, we would like to restrict the iterations to a bounded set around $\boldsymbol{\theta}^0$, which is defined as 
$$
\begin{aligned}
S_\omega:=&\left\{\boldsymbol{\theta}=\left(\mbox{Vec}\left(\boldsymbol{W}_1\right);\cdots; \mbox{Vec}(\boldsymbol{W}_L)\right): \right. \\
&\left. \|\boldsymbol{\theta}-\boldsymbol{\theta}^0\|_2 \leq \omega, 1\leq l\leq L\right\}.
\end{aligned}
$$
In each iteration $t$, the neural Q-learning algorithm obtains a sample of state-action-reward-transition tuple $(s_t,a_t,r_t,s_{t+1},a_{t+1})$ and computes the TD error by  
\begin{equation}
    \label{eqn:td-error-term}
    \Delta_t=Q(\boldsymbol{x}_t;\boldsymbol{\theta}^t)-\Big(r_t+\gamma Q(\boldsymbol{x}_{t+1};\boldsymbol{\theta}^t)\Big)
\end{equation}
with $\boldsymbol{x}_t=\phi(s_t, a_t), \boldsymbol{x}_{t+1}=\phi(s_{t+1},a_{t+1}).$
Then a projected stochastic semi-gradient step is performed to update the weight matrices:
\begin{equation}
\label{eq:semi-grad} 
\boldsymbol{\theta}^{t+1}=\Pi_{S_\omega}\Big(\boldsymbol{\theta}^t-\eta_t \boldsymbol{g}\left(\boldsymbol{\theta}^t\right)\!\!\Big)
\end{equation}
with 
$$
\boldsymbol{g}(\boldsymbol{\theta}^t)= \Delta_t\cdot \nabla_{\boldsymbol{\theta}}  Q(\boldsymbol{x}_t;\boldsymbol{\theta}^t). 
$$
We formally describe the  neural TD learning method in Algorithm \ref{alg:Neural-TD}.

\begin{algorithm}[H]
\caption{Neural Temporal Difference Learning with Markovian Sampling}
\label{alg:Neural-TD}
\begin{algorithmic}
\STATE \textbf{Input:} A learning policy $\pi$, a discount factor $\gamma\in(0,1)$, a sequence of learning rates $\{\eta_t\}_{t\geq0}$, a maximum iteration number $T$, a projection radius $\omega>0$, a Q network with architecture \eqref{eq:nn}.\vspace{0.05cm}
\STATE \textbf{Initialization:} Generate each entry of $\boldsymbol{W}_l^0$ independently from $\mathcal{N}(0, 1)$, for $l=1,2,\cdots, L$, and each entry of $\boldsymbol{b}$ independently from $\text{Unif} \{-1,+1\}$. Generate $s_0\sim\mu, a_0\sim\pi(\cdot|s_0)$.
\vspace{0.05cm}
\FOR{$t=0,1,\cdots,T-1$}
\STATE Sample $(s_t,a_t,r_t,s_{t+1},a_{t+1})$ from the learning policy $\pi$ with $a_{t+1}\sim\pi(\cdot|s_{t+1})$.
\STATE Compute the TD error $\Delta_t$ by \eqref{eqn:td-error-term}.
\STATE Update $\boldsymbol{\theta}^{t+1}$ by the projected stochastic semi-gradient step \eqref{eq:semi-grad}.
\ENDFOR
\vspace{0.05cm}
\STATE \textbf{Output:} $\boldsymbol{\theta}^T$.
\end{algorithmic}
\end{algorithm}

One remark is that, under the non-i.i.d. Markovian sampling setting, the agent is only able to generate a trajectory of samples following some given learning policy $\pi$, which is very common in the offline RL \cite{wu2019behavior, levine2020offline, kostrikov2021offline} where the data trajectories are generated by some learning policy.

In later sections, we will revisit the Algorithm \ref{alg:Neural-TD} and design a novel subspace analysis technique for this method and achieve an improved sample complexity of $\tilde{\mathcal{O}}(\epsilon^{-1})$. Moreover, by replacing the TD error induced by the Bellman operator (\ref{eqn:td-error-term}) with TD error induced by the Bellman optimality operator: $\Delta_t=Q(\boldsymbol{x}_t;\boldsymbol{\theta}^t)-\big(r_t+\gamma \max_{b\in\mathcal{A}}Q(s',b;\boldsymbol{\theta}^t)\big)$, Algorithm \ref{alg:Neural-TD} can be reduced to the neural Q-learning method for finding optimal state-action value $Q^*$. Our analysis for neural TD learning can be extended to the neural Q-learning analogously and obtain the same $\tilde{\mathcal{O}}(\epsilon^{-1})$ sample complexity.

\section{Convergence of Neural Temporal Difference Learning}
\label{section:conv}
\subsection{Basic Settings and Assumptions}
\label{sect:base-setting}
To analyze Algorithm \ref{alg:Neural-TD}, let us first define the local linearization function class of the multi-layer Q network \eqref{eq:nn} at the random initialization $\boldsymbol{\theta}^0$:
\begin{equation}
    \label{defn:local-linearization}
    \mathcal{F}_{\omega, m}:=\left\{\widehat{Q}(\cdot\,;\boldsymbol{\theta})=Q(\cdot\,;\boldsymbol{\theta}^0)+\left<\nabla_{\boldsymbol{\theta}} Q(\cdot\,;\boldsymbol{\theta}^0), \boldsymbol{\theta}-\boldsymbol{\theta}^0\right> \right\}
\end{equation} 
for any $\boldsymbol{\theta} \in S_\omega$.
Consider the MSPBE minimization problem: 
\begin{equation}
\label{eq:mspbe-true}
\min_{\boldsymbol{\theta}\in S_\omega} \mathbb{E}_{\mu,\pi,\mathbb{P}} \left[\left(Q(\boldsymbol{x};\boldsymbol{\theta})-\Pi_{\mathcal{F}_{\omega, m}} \mathcal{T}^\pi Q(\boldsymbol{x};\boldsymbol{\theta})\right)^2\right],
\end{equation}
where $\mu$ is the initial state distribution, $\pi$ is the learning policy, and $\mathbb{P}$ is the transition kernel, the expectation $\mathbb{E}_{\mu,\pi,\mathbb{P}}[\cdot]$ is taken over $s\sim\mu$, $a\sim\pi(\cdot|s)$, and  $s'\sim\mathbb{P}(\cdot|s,a), a'\sim\pi(\cdot|s')$ in $\mathcal{T}^\pi$. Define the set $\Xi_\beta$ as
\begin{align}
\label{defn:mspbe-fixedpoint}
\Xi_\beta := \left\{\boldsymbol{\theta}\in S_\beta :\right.&\left.\widehat{Q}(\boldsymbol{x};\boldsymbol{\theta})=\Pi_{\mathcal{F}_{\omega,m}}\mathcal{T}^\pi\widehat{Q}(\boldsymbol{x};\boldsymbol{\theta}),\right.\nonumber \\
&\left. \forall \boldsymbol{x}=\phi(s,a)  \right\}.
\end{align}
Then the set $\Xi_\omega$ consists of the points $\boldsymbol{\theta}$ with which $\widehat{Q}(\cdot\,;\boldsymbol{\theta})$ forms a fixed point of the projected Bellman operator $\Pi_{\mathcal{F}_{\omega, m}} \mathcal{T}^\pi$ for the problem \eqref{eq:mspbe-true}. By 
Section 4.1 in \cite{cai2023neural}, the fixed point of $\Pi_{\mathcal{F}_{\omega, m}} \mathcal{T}^\pi$ is unique for $\boldsymbol{\theta}\in S_\omega$. Therefore, the following relationship holds  
\begin{equation}
    \label{eq:fixed-point-unique}
    \widehat{Q}(\boldsymbol{x};\boldsymbol{\theta}) = \widehat{Q}(\boldsymbol{x};\boldsymbol{\theta}').
\end{equation}
for $\forall \boldsymbol{x}=\phi(s,a), \,\, \forall (s,a)\in\mathcal{S}\times\mathcal{A}, \,\,\forall \boldsymbol{\theta},\boldsymbol{\theta}'\in\Xi_\beta, \forall \beta\geq \omega.$
Moreover, it is also shown that a point $\boldsymbol{\theta}^*\in\Xi_\omega$ if and only if it satisfies the stationarity condition:
\begin{equation}
\label{eq:optimal}
\mathbb{E}_{\mu,\pi,\mathbb{P}}\left[\widehat{\Delta}\left(\boldsymbol{x}, \boldsymbol{x}' ; \boldsymbol{\theta}^*\right)\big\langle\nabla_{\boldsymbol{\theta}} \widehat{Q}\left(\boldsymbol{x};\boldsymbol{\theta}^* \right), \boldsymbol{\theta}-\boldsymbol{\theta}^*\big\rangle\right] \geq 0,
\end{equation}
where $\widehat{Q}(\cdot\,;\boldsymbol{\theta}^*)\in\mathcal{F}_{\omega,m}$ is a local linearization provided by \eqref{defn:local-linearization} and  $\widehat{\Delta}$ is defined as
$$
\widehat{\Delta}\left(\boldsymbol{x}, \boldsymbol{x}' ; \boldsymbol{\theta}^*\right)=\widehat{Q}(\boldsymbol{x};\boldsymbol{\theta}^*)-\Big(r(s, a)+\gamma \widehat{Q}\left(\boldsymbol{x}';\boldsymbol{\theta}^*\right)\!\Big).$$
Hence people may analyze the gap between $Q^\pi(\cdot)$ and  $Q(\cdot,;\boldsymbol{\theta}^T)$ by first connecting it to $\widehat{Q}(\cdot\,;\boldsymbol{\theta}^*)$. Based on this, \citet{cai2023neural} derived an $\tilde{\mathcal{O}}(\epsilon^{-2})$ sample complexity for the neural TD method. Now we define 
\begin{eqnarray}
\label{eq:feature-matrix}
\Sigma_\pi&=&\mathbb{E}_{\mu, \pi}\left[\nabla_{\boldsymbol{\theta}}Q(\boldsymbol{x};\boldsymbol{\theta}^0)\nabla_{\boldsymbol{\theta}}Q(\boldsymbol{x};\boldsymbol{\theta}^0)^\top\right].
\end{eqnarray}
It is worth noting that the matrix $\Sigma_\pi$ only depends on $\pi$ and $\boldsymbol{\theta}^0$. In the original assumption about \eqref{eq:feature-matrix}, \cite{zou2019finite, xu2020finite} in fact assumed positive definiteness ($\succ0$) of $\Sigma_\pi$, which can be viewed as a generalized version of the positive definite feature covariance matrix assumption in the analysis of linear TD and linear Q-learning, see e.g. \cite{zou2019finite}.  However, in this paper we adopt the following weaker regularity assumption.

\begin{assumption}
\label{as:lamda}
Let $\overline{\sigma_{\min}}(\Sigma_\pi)$ denote the minimum non-zero singular value of the matrix $\Sigma_\pi$, then there exist constants $\lambda_0,m^*>0$ such that $\overline{\sigma_{\min}}(\Sigma_\pi)\geq\lambda_0$ as long as the Q network width $m\geq m^*$.
\end{assumption}
For neural Q function approximation, a sufficient but not necessary condition for Assumption \ref{as:lamda} can be obtained by exploiting the theory of over-parameterized neural networks. Roughly speaking, for a finite MDP with an $L$-layer ReLU Q network, if the feature map satisfies $\phi(s,a)\nparallel\phi(s',a')$ for $\forall(s,a)\neq(s',a')$, the results of \cite{jacot2018neural, allen2019learning, allen2019convergence, cao2019generalization, cao2020generalization} suggest that there exist $\lambda',m^*>0$ such that with high probability $\mathrm{Gram}(\boldsymbol{\theta}_0)\succ\lambda'\!\cdot\!\mathbf{I}$ for networks with width $m\geq m^*.$ Here $\mathrm{Gram}(\boldsymbol{\theta}_0)$ stands for the Gram matrix of the network at the initialization $\boldsymbol{\theta}_0$. A lower bound on $\overline{\sigma_{\min}}(\Sigma_\pi)$ can then be constructed with $\lambda'$, refer to Remark \ref{remark:ntk} in Appendix \ref{appendix:nn}.

Finally, to facilitate the sample complexity analysis under the non-i.i.d. Markovian sampling setting, let us make the following assumption on the fast mixing rate of the MDP sample trajectories, which is widely adopted in the related analysis \cite{zou2019finite, xu2020finite, cai2023neural}. 

\begin{assumption}
\label{as:markov}
We assume that the Markov chain $\left\{s_t\right\}_{t=0,1, \ldots}$ induced by the learning policy $\pi$ and the transition kernel $\mathbb{P}$ is uniformly ergodic with its invariant measure $\mathbb{P}^\pi$. Furthermore, we assume that there are constants $\kappa>0, \rho\in(0,1)$ such that
$$
\sup _{s \in \mathcal{S}} d_{T V}\left(\mathbb{P}\left(s_t \in \cdot \mid s_0=s\right), \mathbb{P}^{\pi}\right) \leq \kappa \rho^t
$$
for all $t\geq0.$
\end{assumption}

Without loss of generality, we also make the following technical assumption, which is not fundamental as opposed to Assumption \ref{as:lamda}, and \ref{as:markov}. 
\begin{assumption}
    \label{as:technical}
    We assume the initial state distribution $\mu$ to be the stationary state distribution under policy $\pi$. 
\end{assumption}
This assumption is in fact very natural. 
Concerning the stationarity of $\mu$, it can always be guaranteed by abandoning the first $\tilde{\mathcal{O}}(t_{\mathrm{mix}})$ samples while Assumption \ref{as:markov} indicates that the mixing time $t_{\mathrm{mix}}=\tilde{\mathcal{O}}(1)$. This assumption guarantees that the operator $\mathcal{T}^\pi$ is $\gamma$-contractive w.r.t. $\|\cdot\|_\mu$ in policy evaluation. Similar assumptions are included in \cite{bhandari2018finite, cai2023neural}.

\subsection{An Improved Complexity of Neural TD Learning}
\label{sect:improve-ana}
To derive the $\tilde{\mathcal{O}}(\epsilon^{-1})$ sample complexity, we rely on the following key observation on subspace decomposition, which is beyond the existing analysis framework. 
 
\begin{proposition}
\label{prop:optimal}
Let $\mathcal{R}(\Sigma_\pi)$ and $\mathcal{K}(\Sigma_\pi)$ denote the range space and kernel space of the matrix $\Sigma_\pi$, respectively. Then for any parameter $\boldsymbol{\theta}\in S_\omega$, there exists $\boldsymbol{\theta}_*$ such that 
$$\boldsymbol{\theta}_*\in\Xi_{2\omega} \qquad\mbox{and}\qquad \boldsymbol{\theta}-\boldsymbol{\theta}_*\in\mathcal{R}(\Sigma_\pi),$$
which also implies that the projections of $\boldsymbol{\theta}$ and $\boldsymbol{\theta}_*$ onto the subspace $\mathcal{K}(\Sigma_\pi)$ are identical.
\end{proposition}

Based on this argument, for the iteration sequence $\{\boldsymbol{\theta}^t\}_{t\geq0}$ generated by Algorithm \ref{alg:Neural-TD}, there exists a sequence $\{\boldsymbol{\theta}^t_*\}_{t\geq0}\subseteq\Xi_{2\omega}$ such that $\{\boldsymbol{\theta}^t-\boldsymbol{\theta}^t_*\}_{t\geq0}\subseteq\mathcal{R}(\Sigma_\pi)$. Therefore, unlike the existing works that analyze $\|\boldsymbol{\theta}^t-\boldsymbol{\theta}^*\|^2$ for some $\boldsymbol{\theta}^*\in\Xi_\omega$, c.f. \cite{cai2023neural,xu2020finite}, we will prove a much faster convergence in $\|\boldsymbol{\theta}^t-\boldsymbol{\theta}^t_*\|^2$. Combined with \eqref{eq:fixed-point-unique}, this further indicates the improved sample complexity in this paper. The proof of this proposition is presented as follows. 
\begin{proof}
For the ease of discussion, let us denote the dimension of weight parameter $\boldsymbol{\theta}$ as $n$. Then we may denote $\Sigma_\pi\in\mathbb{R}^{n\times n}$ and $\boldsymbol{\theta}\in\mathbb{R}^{n}$. First of all let us fix an arbitrary $\bar{\boldsymbol{\theta}}\in\Xi_\omega$, then we may decompose it into two orthogonal components: 
$$\bar{\boldsymbol{\theta}} = \bar{\boldsymbol{\theta}}_{\parallel}+\bar{\boldsymbol{\theta}}_{\bot}\quad\mbox{s.t.}\quad \bar{\boldsymbol{\theta}}_{\parallel}\in\mathcal{R}(\Sigma_\pi)\mbox{ and } \bar{\boldsymbol{\theta}}_{\bot}\in\mathcal{K}(\Sigma_\pi).$$
Similarly, we can decompose the currently considered vector $\boldsymbol{\theta}$ as 
$$\boldsymbol{\theta} = \boldsymbol{\theta}_{\parallel}+\boldsymbol{\theta}_{\bot}\quad\mbox{s.t.}\quad \boldsymbol{\theta}_{\parallel}\in\mathcal{R}(\Sigma_\pi)\mbox{ and } \boldsymbol{\theta}_{\bot}\in\mathcal{K}(\Sigma_\pi).$$
Note that having an arbitrary vector $\boldsymbol{v}\in\mathbb{R}^n$ in the kernel space of $\Sigma_\pi$ means that $\Sigma_\pi\boldsymbol{v}=0$, which further indicates that  
\begin{eqnarray*}
    0 & = &\boldsymbol{v}^\top\Sigma_\pi\boldsymbol{v} \\
    & = &  \boldsymbol{v}^\top\mathbb{E}_{\mu, \pi}\left[\nabla_{\boldsymbol{\theta}}Q(\boldsymbol{x};\boldsymbol{\theta}^0)\nabla_{\boldsymbol{\theta}}Q(\boldsymbol{x};\boldsymbol{\theta}^0)^\top\right]\boldsymbol{v}\\
    & = &  \mathbb{E}_{\mu, \pi}\left[\left\langle\nabla_{\boldsymbol{\theta}}Q(\boldsymbol{x};\boldsymbol{\theta}^0),\boldsymbol{v}\right\rangle^2\right].
\end{eqnarray*}
Therefore, under the measure $(s,a)\sim\mu\times\pi$, 
we have 
\begin{equation}
    \label{eq:zero-as}
    \boldsymbol{v}\in\mathcal{K}(\Sigma_\pi) \quad\Longrightarrow \quad \left\langle\nabla_{\boldsymbol{\theta}}Q(\boldsymbol{x};\boldsymbol{\theta}^0),\boldsymbol{v}\right\rangle = 0\quad a.s.
\end{equation}  
where $a.s.$ stands for almost surely. Therefore,  define $\boldsymbol{\theta}_*=\bar{\boldsymbol{\theta}}_{\parallel} + \boldsymbol{\theta}_{\bot}$, we can check the stationarity condition \eqref{eq:optimal} for $\boldsymbol{\theta}_*$ by establishing:
\begin{align}
    \label{eq:subspace-remaining} &\mathbb{E}_{\mu,\pi,\mathbb{P}}\left[\widehat{\Delta}\left(\boldsymbol{x}, \boldsymbol{x}' ; \boldsymbol{\theta}_*\right)\cdot\big\langle\nabla_{\boldsymbol{\theta}} \widehat{Q}\left(\boldsymbol{x};\boldsymbol{\theta}_* \right), \boldsymbol{\theta}'-\boldsymbol{\theta}_*\big\rangle\right]\\
    & =\mathbb{E}_{\mu,\pi,\mathbb{P}}\left[\widehat{\Delta}\left(\boldsymbol{x}, \boldsymbol{x}' ; \bar{\boldsymbol{\theta}}\right)\cdot\big\langle\nabla_{\boldsymbol{\theta}} \widehat{Q}\left(\boldsymbol{x};\boldsymbol{\theta}_0 \right), \boldsymbol{\theta}'-\bar{\boldsymbol{\theta}}\big\rangle\right]\geq0 \nonumber
\end{align} 
The proof of \eqref{eq:subspace-remaining} is lengthy and is thus moved to Appendix \ref{sect:proof-prop} for succinctness. As a result we have $\boldsymbol{\theta}_*\in\Xi_{2\omega}$ and $\boldsymbol{\theta}-\boldsymbol{\theta}_*=\bar{\boldsymbol{\theta}}_{\parallel}\in\mathcal{K}(\Sigma_\pi)$. Note that for any $\boldsymbol{\theta}\in S_\omega$,
$$
\begin{aligned}
\|\boldsymbol{\theta}_*-\boldsymbol{\theta}^0\|\leq  &\ \|\bar{\boldsymbol{\theta}}_{\parallel}-\boldsymbol{\theta}^0_{\parallel}\|+\|\boldsymbol{\theta}_{\bot}-\boldsymbol{\theta}^0_{\bot}\| \\
\leq&\ \|\bar{\boldsymbol{\theta}}-\boldsymbol{\theta}^0\|+\|\bar{\boldsymbol{\theta}}-\boldsymbol{\theta}^0\|\leq 2\omega, 
\end{aligned}
$$
which completes the proof.\vspace{-0cm}
\end{proof}

Following basic linear algebra analysis, we also have the following proposition. \vspace{0.2cm}

\begin{proposition}
\label{prop:range-strong-convex}
Under Assumption \ref{as:lamda}, suppose the adopted Q network is sufficiently wide so that $m\geq m^*$, then for any $\boldsymbol{\theta}\in \mathcal{R}(\Sigma_\pi)$, we have 
$\boldsymbol{\theta}^\top \Sigma_\pi \boldsymbol{\theta} \geq \lambda_0 \|\boldsymbol{\theta}\|^2_2$.
\end{proposition}

Proposition \ref{prop:optimal} indicates that the variations in the local linearization of Q-function values solely depend on the variations in parameters within the subspace $\mathcal{R}(\Sigma_\pi)$. In the mean while, Proposition \ref{prop:range-strong-convex} indicates that such local linearization is non-singular within $\mathcal{R}(\Sigma_\pi)$. Based on these observations, we can first provide a fast convergence of $\|\boldsymbol{\theta}^t-\boldsymbol{\theta}^t_*\|^2=\mathcal{O}(1/T)$ and then show that $\mathbb{E}\left[\big(\widehat{Q}(\boldsymbol{x};\boldsymbol{\theta}^T)-\widehat{Q}(\boldsymbol{x};\boldsymbol{\theta}^*)\big)^2 \right]=\mathbb{E}\left[\big(\widehat{Q}(\boldsymbol{x};\boldsymbol{\theta}^T)-\widehat{Q}(\boldsymbol{x};\boldsymbol{\theta}^T_*)\big)^2 \right]\leq \mathcal{O}(1/T)$ for any $\boldsymbol{\theta}^*\in\Xi_\omega$. We summarize this result in Theorem \ref{theorem:pi-f} while presenting its proof in Appendix \ref{section:appendix-pi-f}. 

\begin{theorem}
\label{theorem:pi-f}
Suppose Assumptions \ref{as:lamda}, \ref{as:markov} and \ref{as:technical} hold. We set $\omega=\widetilde{C}_1 $ and the learning rate $\eta_t=\frac{1}{2(1-\gamma)\lambda_0 (t+1)}$. If the feature map $\|\phi(s,a)\|=1$ for each state-action pair $(s,a)$ and the network width $m \geq m^*$, then the output $\boldsymbol{\theta}^T$ of Algorithm \ref{alg:Neural-TD} satisfies
\begin{eqnarray*}
&\ &\mathbb{E}\left[\big(\widehat{Q}(\boldsymbol{x};\boldsymbol{\theta}^T)-\widehat{Q}(\boldsymbol{x};\boldsymbol{\theta}^*)\big)^2\mid \boldsymbol{\theta}^0 \right] \\
&\leq& \frac{\widetilde{C}_3(\log T+1)}{(1-\gamma)^2\lambda_0^2T}+ \frac{\widetilde{C}_4m^{-1/2}}{(1-\gamma)\lambda_0}\cdot\sqrt{ \log (T / \delta)}\\
&+&\frac{\widetilde{C}_5\tau^* \left(\log (T / \delta)+1\right) \log T}{ (1-\gamma)^2\lambda_0^2 T},
\end{eqnarray*}
with probability at least $1-2\delta-2L\exp\!\big(\!\!-\widetilde{C}_2 m\big)$, where $\tau^*$ is the mixing time of Markov chain in Assumption \ref{as:markov}, and $\widetilde{C}_1,\cdots,\widetilde{C}_5>0$ are universal constants.
\end{theorem}

Let $Q^\pi$ be the true state-action value function that satisfies the Bellman equation $Q^\pi=\mathcal{T}^\pi Q^\pi$. Then based on the convergence of the local linearization in Theorem \ref{theorem:pi-f}, we establish the global convergence of neural temporal difference learning as Theorem \ref{theorem:total}.

\begin{theorem}
\label{theorem:total}
Suppose the conditions in Theorem \ref{theorem:pi-f} hold. Then the output of Algorithm \ref{alg:Neural-TD} satisfies
\begin{align}
&\mathbb{E}\Big[\big(Q(\phi(s,a);\boldsymbol{\theta}^T)-Q^\pi(s,a)\big)^2\mid \boldsymbol{\theta}^0 \Big]\nonumber\\
\leq&\ \frac{3\mathbb{E}\left[\left(Q^\pi(s,a)-\Pi_{\mathcal{F}_{\omega, m}}Q^\pi(s,a)\right)^2\right]}{(1-\gamma)^2}+\widetilde{C}_6 m^{-1} \nonumber\\
&+\frac{\widetilde{C}_7(\log T+1)}{(1-\gamma)^2\lambda_0^2 T} + \frac{\widetilde{C}_8m^{-1/2}}{(1-\gamma)\lambda_0}\sqrt{\log (T / \delta)}  \\
&+\frac{\widetilde{C}_9\tau^* \left(\log (T / \delta)+1\right) \log T}{(1-\gamma)^2\lambda_0^2 T}\nonumber
\end{align} 
$w.p.\ 1-2 \delta-2L \exp \!\big(\!-\widetilde{C}_2 m \big)$, where $\widetilde{C}_6,\cdots,\widetilde{C}_9>0$ are universal constants.
\end{theorem}

Let $\epsilon_\mathcal{F}:=\frac{3}{(1-\gamma)^2}\mathbb{E}\big[\big(Q^\pi(s,a)-\Pi_{\mathcal{F}_{\omega, m}}Q^\pi(s,a)\big)^2\big]$ be the optimal approximation error of the function class $\mathcal{F}_{\omega, m}$. Then
Theorem \ref{theorem:total} demonstrates that under suitable parameter choices, neural TD learning method identify an approximation error bound of $\mathcal{O}(\epsilon_\mathcal{F}+\epsilon+m^{-\frac{1}{2}})$ within $\tilde{\mathcal{O}}(\epsilon^{-1})$ samples. Existing works include \citet{cai2023neural, xu2020finite, tian2022performance, cayci2023sample} achieve $\tilde{\mathcal{O}}(\epsilon^{-2})$ sample complexity, and \cite{sun2022finite} achieves $\mathcal{O}\big(\epsilon^{-\frac{2}{2-a}}\big), a\in(0,1]$ with additional assumptions. 

Following a similar analysis while adopting an additional regularity assumption on the matrix $\Sigma_\pi$, one can further extend the above analysis to the neural Q-learning by substituting the Bellman operator with the Bellman optimality operator. A similar $\mathcal{O}(\epsilon^{-1})$ sample complexity can still be achieved, which is relegated to  Appendix for succinctness.

\section{Convergence of Minimax Neural Q-Learning}
\label{section:minmax-ql-conv}
A two-player zero-sum Markov game \cite{littman1994markov, bowling2001rational, perolat2018actor}, as a simple variant of MDP, is defined as a six-tuple $\mathcal{M}=(\mathcal{S}, \mathcal{A}_1, \mathcal{A}_2, \mathbb{P}, r, \gamma)$. Here $\mathcal{S}$ is state space, $\mathcal{A}_1$ and $\mathcal{A}_2$ are the action space of the first and second player, respectively, $\mathbb{P}:\mathcal{A}_1\times\mathcal{A}_2\rightarrow \mathcal{P}(\mathcal{S})$ is the transition probability, $r:\mathcal{S}\times\mathcal{A}_1\times\mathcal{A}_2\rightarrow \mathbb{R}$ is the reward function and $\gamma$ is the discounted factor. At time $t$, player 1 and player 2 take actions ($a^1_t\in\mathcal{A}_1$ and $a^2_t\in\mathcal{A}_2$) simultaneously. Player 1 obtains the reward $r(s_t, a^1_t, a^2_t)$. while player 2 obtains $-r(s_t, a^1_t, a^2_t)$. The goal of the two players is to maximize their cumulative rewards respectively. For a policy pair $(\pi_1, \pi_2)$, we can define the state-action value function as follows:
$$
\begin{aligned}
Q^{\pi_1, \pi_2}(s,a^1,a^2)=\ &\mathbb{E}_{\pi_1,\pi_2}\left[\sum_{t=0}^\infty \gamma^t\cdot r(s_t, a^1_t,a^2_t)\mid s_0=s,\right.\\
&\left.a^1_0=a^1, a^2_0=a^2\right],\ \forall s,a^1,a^2.
\end{aligned}
$$
The optimal state-action value function $Q^*$ is defined as 
\begin{eqnarray*}
Q^*(s,a^1,a^2)&=&\max_{\pi_1}\min_{\pi_2} Q^{\pi_1, \pi_2}(s,a^1,a^2)\\
&=&\min_{\pi_2}\max_{\pi_1} Q^{\pi_1, \pi_2}(s,a^1,a^2).
\end{eqnarray*}
We denote the optimal policy pair $\pi^*=\{\pi_1^*, \pi_2^*\}$ if $Q^*(s,a^1,a^2)=Q^{\pi_1^*, \pi_2^*}$. Moreover, the Minimax Bellman operator $\mathcal{H}$ for the Markov game is defined as
$$
\begin{aligned}
\mathcal{H} Q(s,a^1,a^2)=\ &r(s,a^1,a^2)+\gamma \mathbb{E}\left[\min_{b^{1}}\max_{b^{2}} Q(s', b^{1},b^{2})\right. \mid \\
& \left.s'\sim \mathbb{P}(\cdot\mid s,a^1,a^2)\right], \,\,\forall s,a^1,a^2.
\end{aligned}
$$
Thus $\mathcal{H}Q^*=Q^*$. Let the feature map $\boldsymbol{x}=\phi(s,a^1,a^2)$ and $\pi=\{\pi^1, \pi^2\}$ be a given learning policy for players 1 and 2. Assume that $\{s_t, a^1_t,a^2_t,r_t\}_{t=0}^{T}$ is a sampled trajectory of states, actions and rewards obtained from the environment using policy $\pi$. 
Let us recall the definition of the local linearization function class $\mathcal{F}_{\omega,m}$ introduced in \eqref{defn:local-linearization}. Consider the 
MSPBE minimization problem with multi-layer neural network approximation:
\begin{equation*}
\min_{\boldsymbol{\theta}\in S_\omega} \mathbb{E}_{\mu,\pi,\mathbb{P}} \left[\left(Q(\boldsymbol{x};\boldsymbol{\theta})- \Pi_{\mathcal{F}_{\omega,m}}\mathcal{H}Q(\boldsymbol{x};\boldsymbol{\theta})\right)^2\right].
\end{equation*}
To solve this problem, we still adopt the projected stochastic semi-gradient iteration method is provided described by \eqref{eq:semi-grad}, that is,
\begin{equation}
\label{eq:semi-grad-minimax-ql}
\boldsymbol{\theta}^{t+1}=\Pi_{S_\omega}\Big(\boldsymbol{\theta}^t-\eta_t \boldsymbol{g}\left(\boldsymbol{\theta}^t\right)\!\!\Big),
\end{equation}
while redefining the stochastic semi-gradient estimator $\boldsymbol{g}(\boldsymbol{\theta}^t)$ as 
$$
\boldsymbol{g}(\boldsymbol{\theta}^t)= \Delta\left(s_t, a^1_t, a^2_t, s_{t+1} ; \boldsymbol{\theta}^t\right)\cdot \nabla_{\boldsymbol{\theta}}  Q(\boldsymbol{x}_t;\boldsymbol{\theta}^t), 
$$
where $\boldsymbol{x}_t:=\phi(s_t,a^1_t,a^2_t)$ and
\begin{equation}
\label{eq:td-error-minimax-ql}
\begin{aligned}
\Delta\left(s_t, a^1_t,\right.&\!\!\!\left. a^2_t, s_{t+1} ; \boldsymbol{\theta}^t\right)=Q(\boldsymbol{x}_t;\boldsymbol{\theta}^t)-\Big(r(s_t,a^1_t,a^2_t)+\Big.\\
&\Big.\gamma \max_{b^1\in\mathcal{A}_1} \min_{b^2\in\mathcal{A}_2} Q(\phi(s_{t+1},b^1,b^2);\boldsymbol{\theta}^t) \Big).
\end{aligned}
\end{equation}
Now we redefine the function class $\mathcal{F}_{\omega, m}$ as a collection of all local linearization of $Q(\boldsymbol{x};\boldsymbol{\theta})$ at the initial point $\boldsymbol{\theta}^0$:
\begin{align*}
\mathcal{F}_{\omega, m}=\ &\left\{\widehat{Q}(\boldsymbol{x};\boldsymbol{\theta})=Q(\boldsymbol{x};\boldsymbol{\theta}^0)+\left<\nabla_{\boldsymbol{\theta}} Q(\boldsymbol{x};\boldsymbol{\theta}^0), \right.\right.\\
&\left.\left.\boldsymbol{\theta}-\boldsymbol{\theta}^0\right>,\ \boldsymbol{\theta} \in S_\omega\right\}.
\end{align*}
To analyze this method, for any $\beta>0$, we redefine the set $\Xi_\beta$ introduced in \eqref{defn:mspbe-fixedpoint} by replacing the Bellman operator $\mathcal{T}^\pi$ with the Minimax Bellman operator $\mathcal{H}$. Similar to \eqref{eq:fixed-point-unique}, we still have a point $\boldsymbol{\theta}^*\in \Xi_\omega$ if and only if 
$$
\begin{aligned}
\mathbb{E}_{\mu,\pi,\mathbb{P}}&\left[\widehat{\Delta}\left(s, a^1, a^2, s^{\prime} ; \boldsymbol{\theta}^*\right)\big\langle\nabla_{\boldsymbol{\theta}} \widehat{Q}\left(\phi(s,a^1, a^2);\boldsymbol{\theta}^* \right), \big.\right.\\
& \big.\left.\boldsymbol{\theta}-\boldsymbol{\theta}^*\big\rangle\right] \geq 0,
\end{aligned}
$$
where $\widehat{Q}(\cdot\,;\boldsymbol{\theta}^*)\in\mathcal{F}_{\omega,m}$, and $\widehat{\Delta}\left(s, a^1, a^2, s^{\prime} ; \boldsymbol{\theta}\right)$ has the same structure as $\Delta\left(s, a^1, a^2, s^{\prime} ; \boldsymbol{\theta}\right)$ expect that the function $Q(\cdot;\boldsymbol{\theta})$ is replaced by $\widehat{Q}(\cdot;\boldsymbol{\theta})$.

Unlike the neural temporal difference learning method that aims at evaluating the state-action values of a fixed learning policy. The Minimax Bellman operator significantly sophisticates the analysis. Let us redefine the feature covariance matrix $\Sigma_\pi$ with respect to the learning policy $\pi=\{\pi^1,\pi^2\}$, that is
\begin{eqnarray*}
\Sigma_\pi&=&\mathbb{E}_{\pi}\left[\nabla_{\boldsymbol{\theta}}Q(s,a^1, a^2;\boldsymbol{\theta}^0)\nabla_{\boldsymbol{\theta}}Q(s,a^1, a^2;\boldsymbol{\theta}^0)^\top\right].
\end{eqnarray*}
Let the actions $(a^1_{\boldsymbol{\theta}}, a^2_{\boldsymbol{\theta}})$ satisfies $\left<\nabla_{\boldsymbol{\theta}}Q(s,a^1_, a^2;\boldsymbol{\theta}^0), \boldsymbol{\theta}\right>=\max_{a^1\in\mathcal{A}_1}\min_{a^2\in\mathcal{A}_2} \left<\nabla_{\boldsymbol{\theta}}Q(s,a^1, a^2;\boldsymbol{\theta}^0), \boldsymbol{\theta}\right>$. For each parameter pair $\boldsymbol{\theta}_s=(\boldsymbol{\theta}_1, \boldsymbol{\theta}_2)$, we define the action pair $(a^1_{\boldsymbol{\theta}_s}, a^2_{\boldsymbol{\theta}_s})$ that satisfies
$$
\begin{aligned}
(a^1_{\boldsymbol{\theta}_s}, a^2_{\boldsymbol{\theta}_s})=\ &{\arg\max}_{(a^1,a^2)\in \left\{\left(a^1_{\boldsymbol{\theta}_1}, a^2_{\boldsymbol{\theta}_2}\right), \left(a^1_{\boldsymbol{\theta}_2}, a^2_{\boldsymbol{\theta}_1}\right)\right\}} \\
&\Big\{ \left|\left<\nabla_{\boldsymbol{\theta}} Q(s,a^1,a^2;\boldsymbol{\theta}^0), \boldsymbol{\theta}_1-\boldsymbol{\theta}_2 \right>\right| \Big\}.
\end{aligned}
$$
Then for any $\boldsymbol{\theta}_1,\boldsymbol{\theta}_2$, the minimax feature covariance matrix is defined as follows:
$$
\begin{aligned}
\Sigma^*_\pi(\boldsymbol{\theta}_1,\boldsymbol{\theta}_2)=\ &\mathbb{E}_{\pi}\left[\nabla_{\boldsymbol{\theta}}Q(s,a^1_{\boldsymbol{\theta}_s}, a^2_{\boldsymbol{\theta}_s};\boldsymbol{\theta}^0)\right. \\
& \left. \nabla_{\boldsymbol{\theta}}Q(s,a^1_{\boldsymbol{\theta}_s}, a^2_{\boldsymbol{\theta}_s};\boldsymbol{\theta}^0)^\top\right].
\end{aligned}
$$

\begin{figure*}[htbp]
\begin{center}
\centering
\subfigure{
\subfigure{
\includegraphics[width=0.24\textwidth,height=0.192\textwidth]{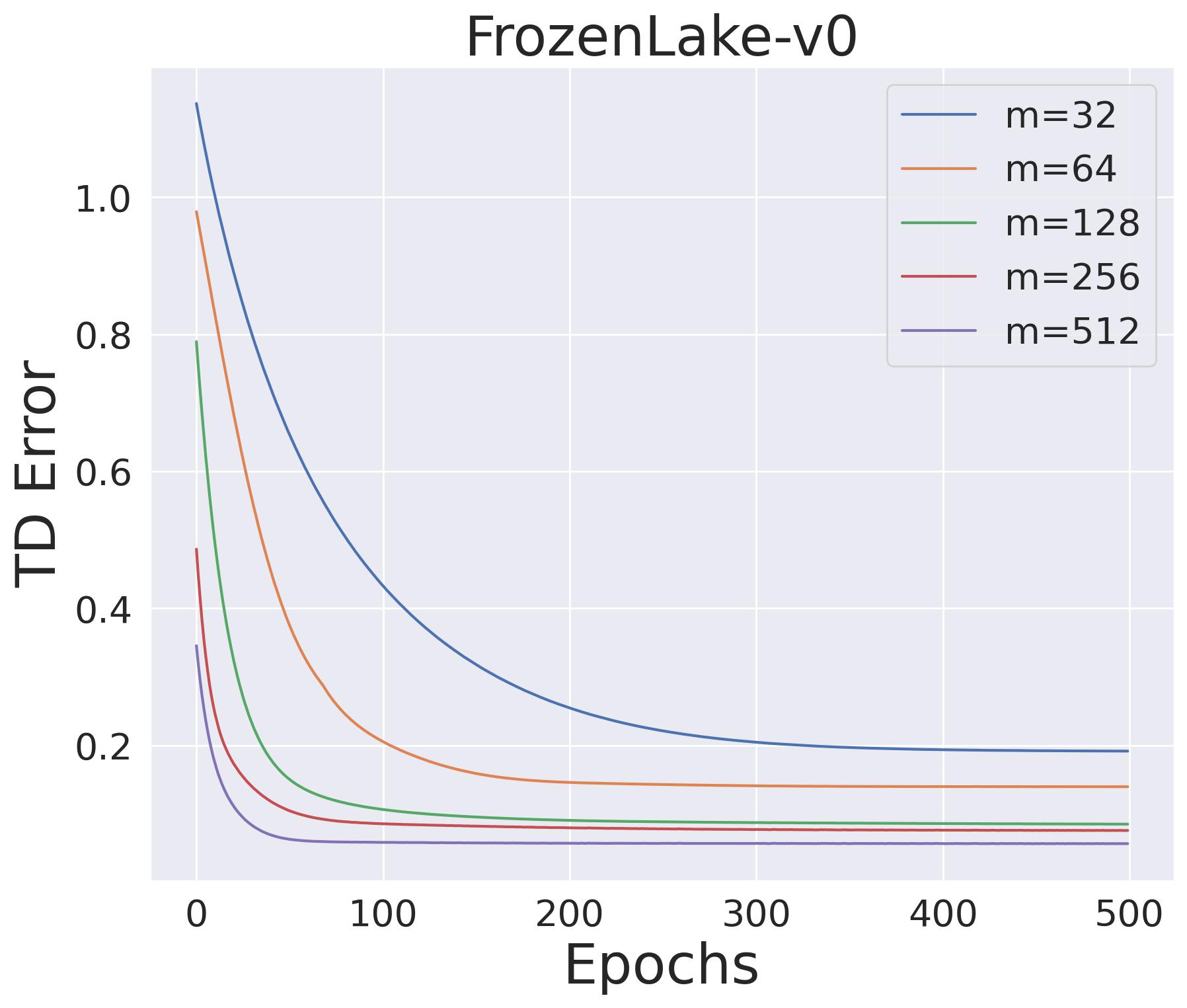}
}
\hspace{-3mm}
\subfigure{
\includegraphics[width=0.24\textwidth,height=0.192\textwidth]{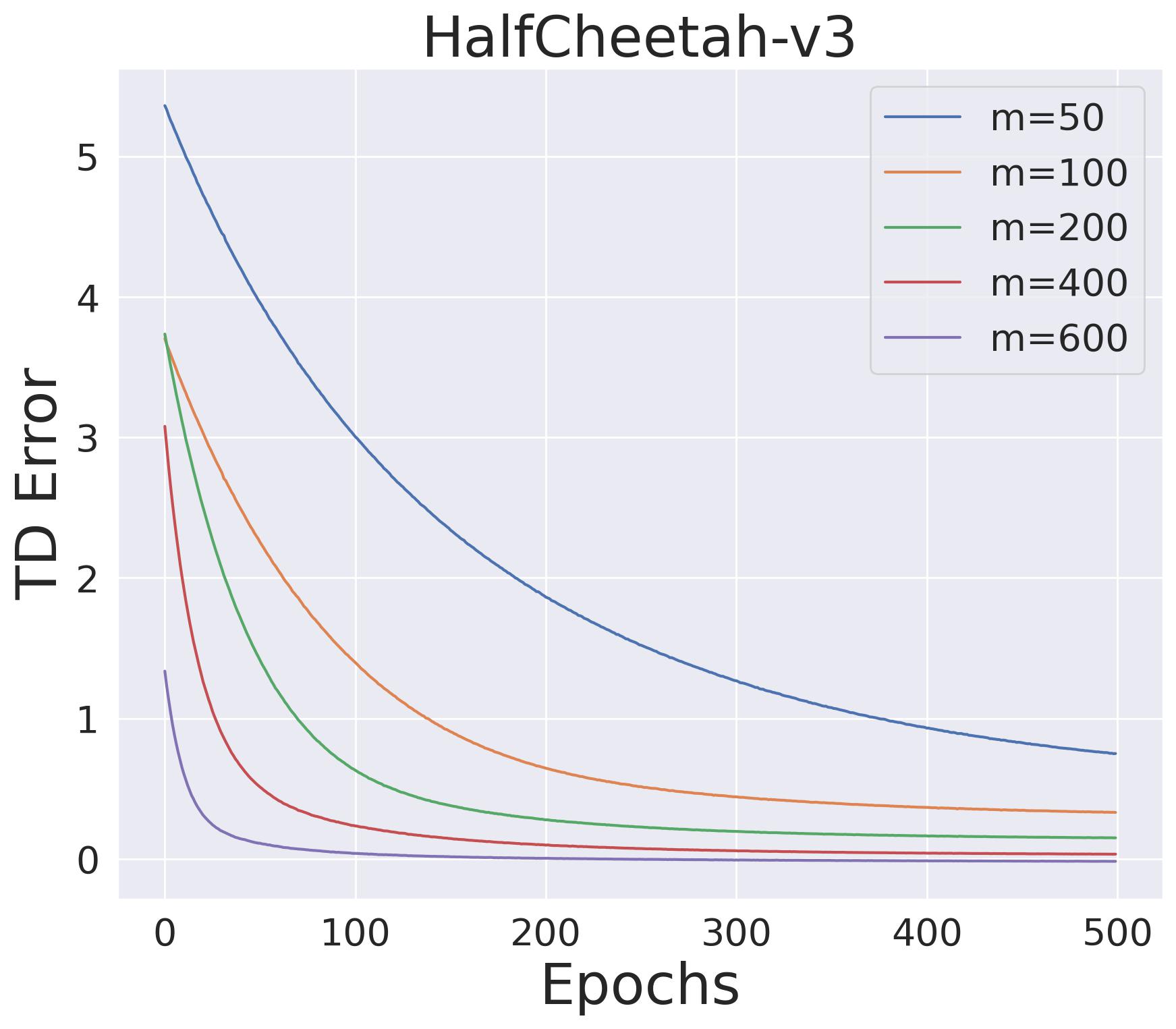}
}
\hspace{-3mm}
\includegraphics[width=0.24\textwidth,height=0.192\textwidth]{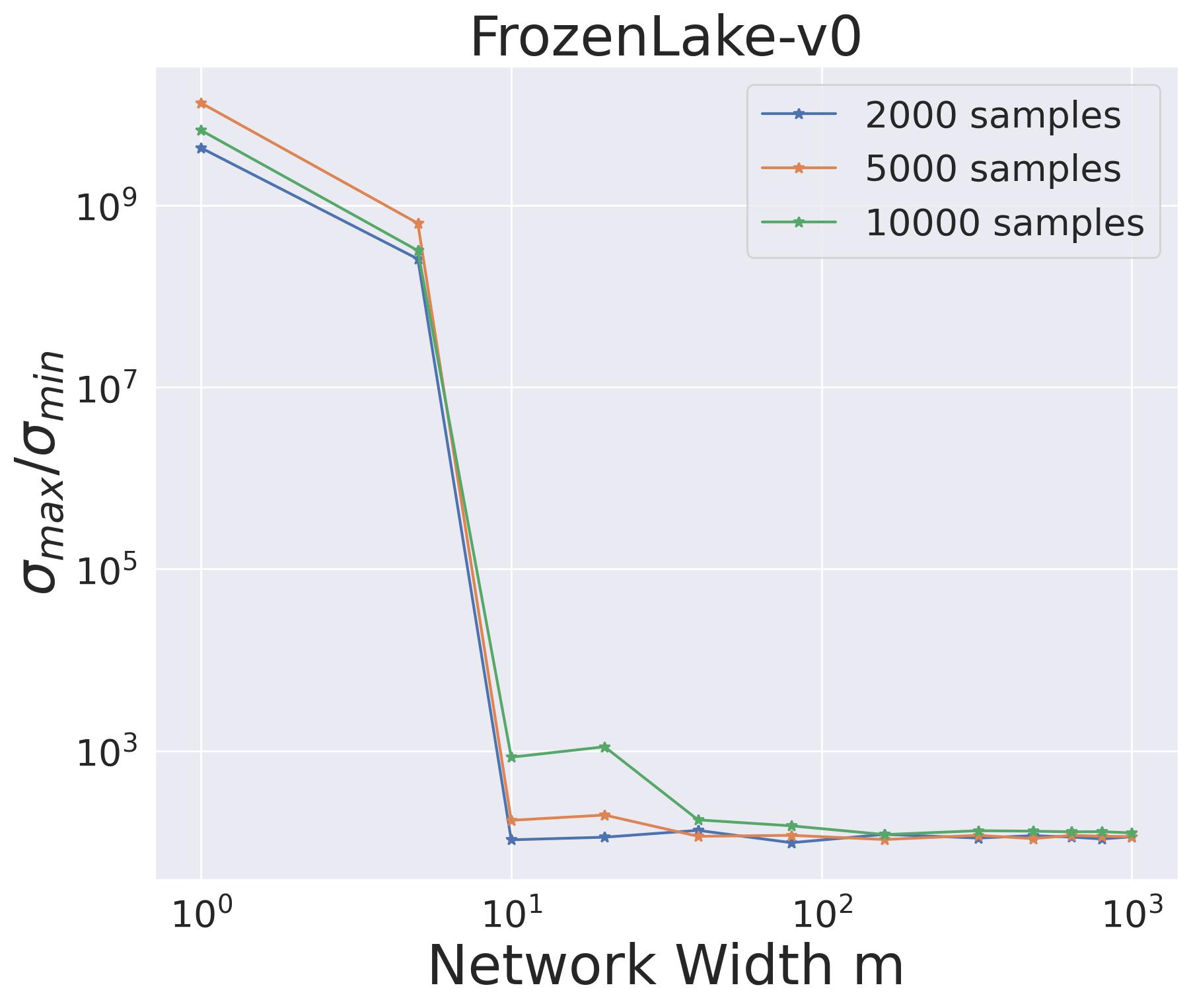}
}
\hspace{-3mm}
\subfigure{
\includegraphics[width=0.24\textwidth,height=0.192\textwidth]{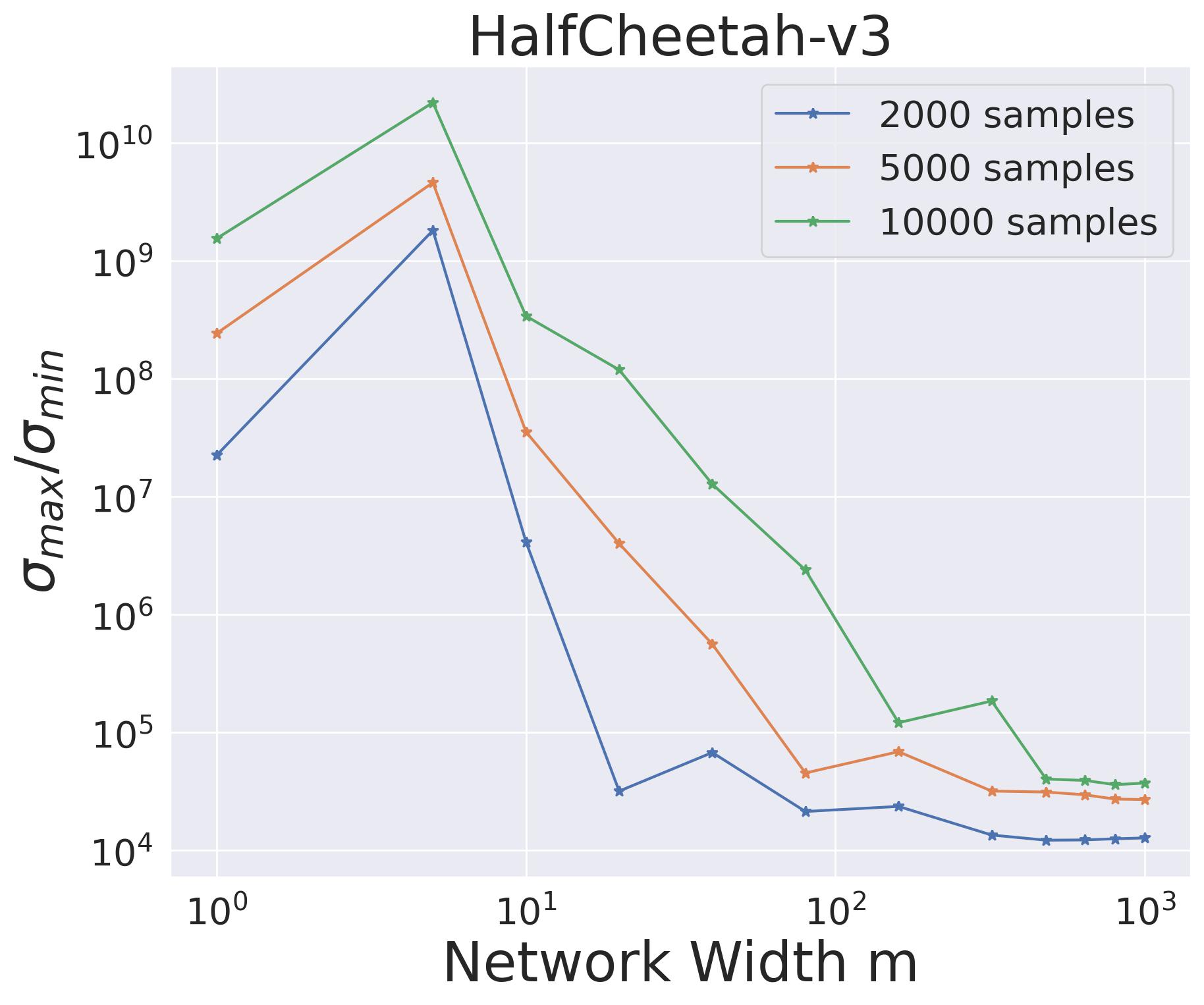}
}
\caption{Training curves and the ratio of the largest and smallest non-zero singular values of $\Sigma_\pi$ over different network widths $m$.} 
\label{fig:td-learning}
\end{center} 
\end{figure*}

\begin{assumption}
\label{as:minimax-regularity}
For any $\boldsymbol{\theta}^1, \boldsymbol{\theta}^2$, there exists a constant $\nu\in(0,1)$ such that $(1-\nu)^2\Sigma_\pi-\gamma^2\Sigma_\pi^*(\boldsymbol{\theta}_1,\boldsymbol{\theta}_2)\succeq 0$.
\end{assumption}
Note the original version of this assumption in \cite{zou2019finite} in fact requires a strict positive definite condition: $((1-\nu)^2\Sigma_\pi-\gamma^2\Sigma_\pi^*(\boldsymbol{\theta}_1,\boldsymbol{\theta}_2)\succ 0$. Under this additional assumption, \cite{zou2019finite} obtained an $\tilde{\mathcal{O}}(\epsilon^{-1})$ sample complexity for minimax Q-learning with linear function approximation. With the help of our subspace analysis technique, in this paper, we relax it to the positive semi-definiteness ($\succeq0$). Now we are ready to state our result for minimax neural Q-learning.
\begin{theorem}
\label{theorem:minimax-pi-f}
Suppose Assumptions \ref{as:lamda}, \ref{as:markov} and \ref{as:minimax-regularity} hold. We set $\omega=\widetilde{C}_1 $ and the learning rate $\eta_t=\frac{1}{2\nu\lambda_0 (t+1)}$. If the feature map $\|\phi(s,a^1,a^2)\|=1$ for each state-action pair $(s,a^1,a^2)$ and the network width $m \geq m^*$, then the output $\boldsymbol{\theta}^T$ of neural minimax Q-learning Algorithm \ref{alg:minimax-q-learning} satisfies
\begin{eqnarray*}
&\ &\mathbb{E}\left[\left(\widehat{Q}(\boldsymbol{x};\boldsymbol{\theta}^T)-\widehat{Q}(\boldsymbol{x};\boldsymbol{\theta}^*)\right)^2\mid \boldsymbol{\theta}^0 \right] \\
&\leq& \frac{\widetilde{C}_3(\log T+1)}{\nu^2\lambda_0^2T}+ \frac{\widetilde{C}_4m^{-1/2}}{\nu\lambda_0}\cdot\sqrt{ \log (T / \delta)} \\
&\ &+\frac{\widetilde{C}_5\tau^* \left(\log (T / \delta)+1\right) \log T}{ \nu^2\lambda_0^2 T},
\end{eqnarray*}
with probability at least $1-2\delta-2L\exp{(-\widetilde{C}_2 m)}$, where $\tau^*$ is the mixing time of Markov chain in Assumption \ref{as:markov}, and $\left\{\widetilde{C}_i>0\right\}_{i=1, \ldots, 5}$ are universal constants.
\end{theorem}

Theorem \ref{theorem:minimax-pi-f} establishes a finite-time analysis of $\tilde{\mathcal{O}}(\epsilon^{-1})$-sample complexity for minimax neural Q-learning in terms of the function class $\mathcal{F}_{\omega,m}$. For a more specific description and theorem proof, see Appendix \ref{section:minimax-conv}. To the best of our knowledge, this is the first analysis of minimax Q-learning with neural network function approximation, characterized by a complexity bound of $\tilde{\mathcal{O}}(\epsilon^{-1})$.

\section{Experiments}
\label{sect:exp}
Finally, we construct several experiments over the OpenAI Gym \cite{brockman2016openai} tasks and validate our theoretical findings. We consider a two-layer neural network, as follows:
\begin{equation*}
\label{eq:neural-network}
    Q(s,a;\boldsymbol{\theta}):=\frac{1}{\sqrt{m}}\sum_{r=1}^mb_r\sigma(\boldsymbol{\theta}_r^{\top}\phi(s,a)),
\end{equation*}
where $\sigma(\cdot)$ is ELU activation in this section.
Furthermore, details regarding the initialization and iteration methods for the parameters can be found in Section \ref{sect:prelim}. For all experiments, we generate samples based on a prescribed $\epsilon$-greedy policy with $\epsilon=0.1$. To prevent redundancy in the features $\phi(s,a)$, we employ one-hot encoding for discrete action-state spaces and implement a fixed grid discretization for continuous spaces. when 
both $\phi(s,a)$ and $\phi(s',a')$ belong to the same one-hot encoding or grid, we treat them as the same sample point. Our investigation into the impact of network width on the TD learning algorithm will be conducted from two perspectives:
(i) examining whether the network width $m$ is correlated with the TD error, and (ii) exploring the existence of constants $m^*$ and $\lambda_0$ that satisfy Assumption \ref{as:lamda}.

The four subfigures in Figure \ref{fig:td-learning} represent two types of environments: one with a discrete state space and the other with a continuous state space. The first two subfigures depict the convergence performance of the TD algorithm at different network widths. We generate 2,000 sample points and run for 500 epochs. Notably, as the parameter $m$ increases, the TD algorithm demonstrates faster convergence, resulting in smaller final TD errors. The latter two subfigures illustrate the existence of $m^*$ and $\lambda_0$. Specifically, we compute the largest non-zero singular value $\sigma_{\max}$ and smallest non-zero singular value $\sigma_{\min}$ of the matrix $\Sigma_\pi$. To mitigate the absolute magnitude of $\sigma_{\min}$, we introduce the ratio $r = \sigma_{\max} / \sigma_{\min}$ as a metric to validate Assumption \ref{as:lamda}.
It can be observed that the value of $r$ approaches a constant as $m$ increases for all cases, providing empirical support for the validity of the assumption. 

\section{Conclusion}
We study the finite-time analysis of the TD and Q learning methods with neural network approximation, where the state-action pairs are generated by a given policy under the Markovian sampling. Besides the convergence to the true action-value function except for an inevitable function approximation error, an improved analysis technique is introduced to establish an $\tilde{\mathcal{O}}(\epsilon^{-1})$ complexity for the neural TD and Q learning methods, which improves the existing $\tilde{\mathcal{O}}(\epsilon^{-2})$ complexity. 
For future work, it is also interesting to investigate if the proposed technique can improve the current complexity estimate of the actor-critic methods, which are partially built upon the neural TD methods.  

\section{Acknowledgements}
Dr. Zaiwen Wen is supported in part by the NSFC grant 12331010. Dr. Junyu Zhang is supported in part by the MOE AcRF grant A-0009530-05-00.

\nocite{langley00}

\bibliography{icml2024}
\bibliographystyle{icml2024}

\newpage
\appendix
\onecolumn

\section{Details of Section \ref{section:conv}}
\label{sect:conv-td}
\subsection{Proof of \eqref{eq:subspace-remaining}}
\label{sect:proof-prop}
\begin{eqnarray*} &&\mathbb{E}_{\mu,\pi,\mathbb{P}}\left[\widehat{\Delta}\left(\boldsymbol{x},\boldsymbol{x}' ; \boldsymbol{\theta}_*\right)\cdot\big\langle\nabla_{\boldsymbol{\theta}} \widehat{Q}\left(\boldsymbol{x};\boldsymbol{\theta}_* \right), \boldsymbol{\theta}'-\boldsymbol{\theta}_*\big\rangle\right] \\
& \overset{(i)}{=}& \mathbb{E}_{\mu,\pi}\left[\mathbb{E}_\mathbb{P}\big[\widehat{\Delta}\left(\boldsymbol{x},\boldsymbol{x}' ; \boldsymbol{\theta}_*\right)\big]\cdot\big\langle\nabla_{\boldsymbol{\theta}} \widehat{Q}\left(\boldsymbol{x};\boldsymbol{\theta}_0 \right), \boldsymbol{\theta}'-\bar{\boldsymbol{\theta}}\big\rangle\right]\\
& \overset{(ii)}{=} & \mathbb{E}_{\mu,\pi}\left[\mathbb{E}_\mathbb{P}\big[\widehat{\Delta}\left(\boldsymbol{x},\boldsymbol{x}' ; \bar{\boldsymbol{\theta}}\right)\big]\cdot\big\langle\nabla_{\boldsymbol{\theta}} \widehat{Q}\left(\boldsymbol{x};\boldsymbol{\theta}_0 \right), \boldsymbol{\theta}'-\bar{\boldsymbol{\theta}}\big\rangle\right]\\
&=& \mathbb{E}_{\mu,\pi,\mathbb{P}}\left[\widehat{\Delta}\left(\boldsymbol{x},\boldsymbol{x}' ; \bar{\boldsymbol{\theta}}\right)\cdot\big\langle\nabla_{\boldsymbol{\theta}} \widehat{Q}\left(\boldsymbol{x};\boldsymbol{\theta}_0 \right), \boldsymbol{\theta}'-\bar{\boldsymbol{\theta}}\big\rangle\right] 
\end{eqnarray*}

where (i) is because $\nabla_{\boldsymbol{\theta}} \widehat{Q}\left(\cdot\,;\boldsymbol{\theta}_* \right)=\widehat{Q}\left(\cdot\,;\boldsymbol{\theta}_0 \right)$, the decomposition 
$$\boldsymbol{\theta}'-\boldsymbol{\theta}_* = \boldsymbol{\theta}'-\bar{\boldsymbol{\theta}} + \bar{\boldsymbol{\theta}}-\boldsymbol{\theta}_* = \boldsymbol{\theta}'-\bar{\boldsymbol{\theta}} + (\bar{\boldsymbol{\theta}}_{\bot}-\boldsymbol{\theta}_{\bot})$$
the fact that $(\bar{\boldsymbol{\theta}}_{\bot}-\boldsymbol{\theta}_{\bot})\in \mathcal{K}(\Sigma_\pi)$, and \eqref{eq:zero-as}, (ii) is because $(\bar{\boldsymbol{\theta}}-\boldsymbol{\theta}^*)\in \mathcal{K}(\Sigma_\pi)$.

\subsection{Proof of Theorem \ref{theorem:pi-f}}
\label{section:appendix-pi-f}
\begin{proof}
Recall the definition of the semi-gradient in Section \eqref{eq:semi-grad}. We denote $\bar{\mathbf{g}}(\boldsymbol{\theta})$ as its expectation. Let $\Bar{\mathbf{m}}(\boldsymbol{\theta})$ and $\mathbf{m}(\boldsymbol{\theta})$ also be the corresponding semi-gradients based on the linearized function $\widehat{Q}(\cdot;\boldsymbol{\theta})$, that is,
\begin{eqnarray*}
\mathbf{g}(\boldsymbol{\theta}^t) &=& \Delta(\boldsymbol{x}_t, \boldsymbol{x}_{t+1};\boldsymbol{\theta}^t)\cdot  \nabla_{\boldsymbol{\theta}}  Q(\boldsymbol{x}_t;\boldsymbol{\theta}^t), \quad
\Bar{\mathbf{g}}(\boldsymbol{\theta}^t) \ = \ \mathbb{E}_{\mu, \pi,\mathbb{P}}\left[\mathbf{g}(\boldsymbol{\theta}^t)\right]\\ 
\mathbf{m}(\boldsymbol{\theta}^t) &=& \widehat{\Delta} (\boldsymbol{x}_t, \boldsymbol{x}_{t+1};\boldsymbol{\theta}^t)\cdot\nabla_{\boldsymbol{\theta}}  Q(\boldsymbol{x}_t;\boldsymbol{\theta}^0),
\quad \Bar{\mathbf{m}}(\boldsymbol{\theta}^t)\ =\ \mathbb{E}_{\mu, \pi,\mathbb{P}}\left[\mathbf{m}(\boldsymbol{\theta}^t)\right],
\end{eqnarray*}
where 
\begin{eqnarray*}
\Delta(\boldsymbol{x}_t, \boldsymbol{x}_{t+1};\boldsymbol{\theta}^t)&=& Q(\boldsymbol{x}_t;\boldsymbol{\theta}^t)-\left(r(s_t, a_t)+\gamma \cdot Q(\boldsymbol{x}_{t+1};\boldsymbol{\theta}^t)\right), \\
\widehat{\Delta} (\boldsymbol{x}_t, \boldsymbol{x}_{t+1};\boldsymbol{\theta}^t) &=&\widehat{Q}(\boldsymbol{x}_t;\boldsymbol{\theta}^t)-\left(r(s_t, a_t)+\gamma \cdot \widehat{Q}(\boldsymbol{x}_{t+1};\boldsymbol{\theta}^t)\right).
\end{eqnarray*}
To simplify the notation, let $\Delta_t:=\Delta(\boldsymbol{x}_t, \boldsymbol{x}_{t+1};\boldsymbol{\theta}^t)$ and $\widehat{\Delta}_t:=\widehat{\Delta}(\boldsymbol{x}_t, \boldsymbol{x}_{t+1};\boldsymbol{\theta}^t)$. Recall the definition of the range space $\mathcal{R}(\Sigma_\pi)$ and the kernel space $\mathcal{K}(\Sigma_\pi)$. By Proposition \ref{prop:optimal}, we know that $\boldsymbol{v}_1^\top \boldsymbol{v}_2 = 0$ for any vector $\boldsymbol{v}_1\in \mathcal{R}(\Sigma_\pi), \boldsymbol{v}_2\in \mathcal{K}(\Sigma_\pi)$ thus $\left<\nabla_{\boldsymbol{\theta}} Q(\boldsymbol{x};\boldsymbol{\theta}^0), \boldsymbol{\theta}_{\bot}\right>=0$ for any feature map $\boldsymbol{x}$ and parameter $\boldsymbol{\theta}_\bot\in \mathcal{K}(\Sigma_\pi)$. Then we can decompose $\|\boldsymbol{\theta}^{t+1}-\boldsymbol{\theta}^{t+1}_*\|^2$ as

\begin{eqnarray*}
\|\boldsymbol{\theta}^{t+1}-\boldsymbol{\theta}^{t+1}_*\|^2&=& \left\| \Pi_{S_{2\omega}}\left(\boldsymbol{\theta}^t-\eta_t \mathbf{g}(\boldsymbol{\theta}^t)\right)-\Pi_{S_{2\omega}}\left(\boldsymbol{\theta}^{t+1}_*\right) \right\|^2 \\
&\leq& \|\boldsymbol{\theta}^t-\eta_t \mathbf{g}(\boldsymbol{\theta}^t)-\boldsymbol{\theta}^{t+1}_*\|^2 \\
&=& \|\boldsymbol{\theta}^t-\eta_t \mathbf{g}(\boldsymbol{\theta}^t)-\boldsymbol{\theta}^t_*+\boldsymbol{\theta}^t_*-\boldsymbol{\theta}^{t+1}_*\|^2 \\
&=& \|\boldsymbol{\theta}^{t}-\boldsymbol{\theta}^t_*\|^2+\eta_t^2\|\mathbf{g}(\boldsymbol{\theta}^t)\|^2+\|\boldsymbol{\theta}^t_*-\boldsymbol{\theta}^{t+1}_*\|^2-2\eta_t \left<\boldsymbol{\theta}^t-\boldsymbol{\theta}^t_*, \mathbf{g}(\boldsymbol{\theta}^t)\right> \\
&\ & -2\eta_t \left<\boldsymbol{\theta}^t_*-\boldsymbol{\theta}^{t+1}_*, \mathbf{g}(\boldsymbol{\theta}^t)\right>+2\eta_t \left<\boldsymbol{\theta}^t-\boldsymbol{\theta}^t_*, \boldsymbol{\theta}^t_*-\boldsymbol{\theta}^{t+1}_*\right> \\
&\overset{(i)}{=}& \|\boldsymbol{\theta}^{t}-\boldsymbol{\theta}^t_*\|^2-2\eta_t \left<\boldsymbol{\theta}^t-\boldsymbol{\theta}^t_*, \mathbf{g}(\boldsymbol{\theta}^t)\right> \\
&\ &-2\eta_t \left<\boldsymbol{\theta}^t_*-\boldsymbol{\theta}^{t+1}_*, \mathbf{g}(\boldsymbol{\theta}^t)-\mathbf{m}(\boldsymbol{\theta}^t)\right> + \eta_t^2\|\mathbf{g}(\boldsymbol{\theta}^t)\|^2,
\end{eqnarray*}
where (i) follows 
$$
\left<\boldsymbol{\theta}^t-\boldsymbol{\theta}^t_*, \boldsymbol{\theta}^t_*-\boldsymbol{\theta}^{t+1}_*\right>=\left<\boldsymbol{\theta}^t_{\parallel}-\boldsymbol{\theta}^*_{\parallel}, \boldsymbol{\theta}^t_{\bot}-\boldsymbol{\theta}^{t+1}_{\bot}\right>=0,
$$
and
\begin{eqnarray*}
\left<\boldsymbol{\theta}^t_*-\boldsymbol{\theta}^{t+1}_*, \mathbf{m}(\boldsymbol{\theta}^t)\right>&=&\left<\boldsymbol{\theta}^t_{\bot}-\boldsymbol{\theta}^{t+1}_{\bot}, \mathbf{m}(\boldsymbol{\theta}^t)\right> \\
&=& \widehat{\Delta}_t\cdot \left<\boldsymbol{\theta}^t_{\bot}-\boldsymbol{\theta}^{t+1}_{\bot}, \nabla Q(\boldsymbol{x}_t;\boldsymbol{\theta}^0)\right>\ =\ 0.
\end{eqnarray*}
Recall the stationarity condition \eqref{eq:optimal}, for any $t\in\{1,2,\cdots,T\}$, 
\begin{eqnarray*}
0&\leq& \mathbb{E}_{\mu,\pi,\mathbb{P}}\left[\widehat{\Delta}\left(\boldsymbol{x}_t,\boldsymbol{x}_{t+1}; \boldsymbol{\theta}^*\right)\big\langle\nabla_{\boldsymbol{\theta}} \widehat{Q}\left(\boldsymbol{x}_t;\boldsymbol{\theta}^* \right), \boldsymbol{\theta}^t-\boldsymbol{\theta}^*\big\rangle\right] \\
&=&\big\langle\mathbb{E}_{\mu,\pi,\mathbb{P}}\left[\widehat{\Delta}\left(\boldsymbol{x}_t,\boldsymbol{x}_{t+1}; \boldsymbol{\theta}^t_*\right)\nabla_{\boldsymbol{\theta}} \widehat{Q}\left(\boldsymbol{x}_t;\boldsymbol{\theta}^t_* \right)\right], \boldsymbol{\theta}^t-\boldsymbol{\theta}^*\big\rangle  \\
&\overset{(i)}{=}& \big\langle\mathbb{E}_{\mu,\pi,\mathbb{P}}\left[\widehat{\Delta}\left(\boldsymbol{x}_t,\boldsymbol{x}_{t+1}; \boldsymbol{\theta}^t_*\right)\nabla_{\boldsymbol{\theta}} \widehat{Q}\left(\boldsymbol{x}_t;\boldsymbol{\theta}^t_* \right)\right], \boldsymbol{\theta}^t-\boldsymbol{\theta}^t_*\big\rangle \ =\ \left<\bar{\mathbf{m}}(\boldsymbol{\theta^t_*}), \boldsymbol{\theta}^t-\boldsymbol{\theta}^t_*\right>, 
\end{eqnarray*}
where (i) is the same as the proof in Section \ref{sect:proof-prop}.
Therefore, 
\begin{eqnarray}
\label{eq:err-decompose}
&\ & \|\boldsymbol{\theta}^{t+1}-\boldsymbol{\theta}^{t+1}_*\|^2 \nonumber \\
&\leq &\|\boldsymbol{\theta}^{t}-\boldsymbol{\theta}^t_*\|^2-2\eta_t \left<\boldsymbol{\theta}^t-\boldsymbol{\theta}^t_*, \mathbf{g}(\boldsymbol{\theta}^t)\right>-2\eta_t \left<\boldsymbol{\theta}^t_*-\boldsymbol{\theta}^{t+1}_*, \mathbf{g}(\boldsymbol{\theta}^t)-\mathbf{m}(\boldsymbol{\theta}^t)\right> + \eta_t^2\|\mathbf{g}(\boldsymbol{\theta}^t)\|^2 \nonumber \\
&=&\|\boldsymbol{\theta}^{t}-\boldsymbol{\theta}^t_*\|^2-2\eta_t \left<\boldsymbol{\theta}^t-\boldsymbol{\theta}^t_*, \mathbf{g}(\boldsymbol{\theta}^t)-\mathbf{m}(\boldsymbol{\theta}^t)\right>-2\eta_t \left<\boldsymbol{\theta}^t-\boldsymbol{\theta}^t_*, \mathbf{m}(\boldsymbol{\theta}^t)-\Bar{\mathbf{m}}(\boldsymbol{\theta}^t)\right> \nonumber \\
&\ & -2\eta_t \left<\boldsymbol{\theta}^t-\boldsymbol{\theta}^t_*, \Bar{\mathbf{m}}(\boldsymbol{\theta}^t)\right>-2\eta_t \left<\boldsymbol{\theta}^t_*-\boldsymbol{\theta}^{t+1}_*, \mathbf{g}(\boldsymbol{\theta}^t)-\mathbf{m}(\boldsymbol{\theta}^t)\right> + \eta_t^2\|\mathbf{g}(\boldsymbol{\theta}^t)\|^2 \nonumber\\
&\overset{(i)}{\leq}&\|\boldsymbol{\theta}^{t}-\boldsymbol{\theta}^t_*\|^2+ \eta_t^2\underbrace{\|\mathbf{g}(\boldsymbol{\theta}^t)\|^2}_{\mbox{I}_1\mbox{:\ Gradient Bound}}-2\eta_t\underbrace{\left<\boldsymbol{\theta}^t-\boldsymbol{\theta}^{t+1}_*, \mathbf{g}(\boldsymbol{\theta}^t)-\mathbf{m}(\boldsymbol{\theta}^t)\right>}_{\mbox{I}_2\mbox{:\ Gradient Gap}} \nonumber \\
&\ &-2\eta_t\underbrace{\left<\boldsymbol{\theta}^t-\boldsymbol{\theta}^t_*, \mathbf{m}(\boldsymbol{\theta}^t)-\Bar{\mathbf{m}}(\boldsymbol{\theta}^t)\right>}_{\mbox{I}_3\mbox{:\ Markov Sampling Error}} 
-2\eta_t\underbrace{\left<\boldsymbol{\theta}^t-\boldsymbol{\theta}^t_*, \Bar{\mathbf{m}}(\boldsymbol{\theta}^t)- \Bar{\mathbf{m}}(\boldsymbol{\theta}^t_*)\right>}_{\mbox{I}_4\mbox{:\ Gradient Decent}},
\end{eqnarray}
where (i) follows $\left<\boldsymbol{\theta}^t-\boldsymbol{\theta}^t_*,  \Bar{\mathbf{m}}(\boldsymbol{\theta}^t_*)\right>\geq 0$ for any $0\leq t\leq T-1$.

\ 

Next, we analyze the upper bounds of $\mathbf{I}_1$, $\mathbf{I}_2$, $\mathbf{I}_3$ and $\mathbf{I}_4$ item by item. To simplify the notation, let $\left\{C_i>0\right\}_{i=1, \ldots, 7}$ be universal constants in this section.
We set $\omega=C_1$ and $\delta \in(0,1)$. 
By Lemma \ref{lemma:I1-and-I2}, we have
\begin{equation}
\label{eq:I1}
\|\boldsymbol{g}(\boldsymbol{\theta}^t)\|^2 \leq C_2\sqrt{\log (T / \delta)}
\end{equation}
and
\begin{eqnarray}
\label{eq:I2}
\mathbb{E}_{\mu,\pi, \mathbb{P}}\left[\left|\left<\mathbf{g}\left(\boldsymbol{\theta}^t\right)-\mathbf{m}\left(\boldsymbol{\theta}^t\right), \boldsymbol{\theta}^t-\boldsymbol{\theta}^{t+1}_*\right>\right|\mid\boldsymbol{\theta}^0\right]  &\leq &\left(C_3 \omega m^{-\frac{1}{2}}\sqrt{ \log (T / \delta)}+C_4 m^{-\frac{1}{2}}\right)\left\| \boldsymbol{\theta}^t-\boldsymbol{\theta}^{t+1}_*\right\|\nonumber \\
&\overset{(i)}{\leq}& C_5m^{-1/2}\sqrt{\log (T / \delta)},
\end{eqnarray}
with probability at least $1-2\delta-2\exp{(-C_6 m)}$, 
where (i) follows $\omega=C_1$ and 
\begin{eqnarray*}
\|\theta^t-\theta^{t+1}_*\|&\leq& \|\theta^t-\theta^0\|+\|\theta^0-\theta^{0}_*\|+\|\theta^0_*-\theta^{t+1}_*\|\\
&=&\|\theta^t-\theta^0\|+\|\theta^0-\theta^{0}_*\|+\|\theta^0_{\bot}-\theta^{t+1}_{\bot}\|  \\
&\leq&\|\theta^t-\theta^0\|+\|\theta^0-\theta^{0}_*\|+\|\theta^0-\theta^{t+1}\|  \ \leq \  3\omega.
\end{eqnarray*}

Thus \eqref{eq:I1} and \eqref{eq:I2} provide upper bounds on $\mathbf{I}_1$ and $\mathbf{I}_2$, respectively. The next lemma provides an estimate of the Markov sampling error.
\begin{lemma}
\label{lemma:I3}
Suppose the learning rate sequence $\left\{\eta_0, \eta_1, \ldots, \eta_T\right\}$ is non-increasing. Under Assumption \ref{as:markov}, it holds that 
\begin{equation}
\label{eq:I3}
\mathbb{E}_{\mu,\pi,\mathbb{P}}\left[\left\langle\mathbf{m}\left(\boldsymbol{\theta}^t\right)-\overline{\mathbf{m}}\left(\boldsymbol{\theta}^t\right), \boldsymbol{\theta}^t-\boldsymbol{\theta}^*\right\rangle \mid \boldsymbol{\theta}^0\right] \leq C_7\left(\log (T / \delta)+C_1^2\right) \tau^* \eta_{\max \left\{0, t-\tau^*\right\}},
\end{equation}
for any fixed $t \leq T$, where 
$$\tau^*=\min \left\{t=0,1,2, \ldots \mid \kappa \rho^t \leq \eta_T\right\}
$$
is the mixing time of the Markov chain $\left\{s_t, a_t\right\}_{t=0,1, \ldots .}$.
\end{lemma}

\begin{proof}
We adopt the proof framework outlined in Lemma 6.2 of \citet{xu2020finite}. However, variations in the neural network settings lead to differences in the norms of gradients and parameters, thereby resulting in slight variations in the results. Thereby we have
$$
\mathbb{E}_{\mu,\pi,\mathbb{P}}\left[\left\langle\mathbf{m}\left(\boldsymbol{\theta}^t\right)-\overline{\mathbf{m}}\left(\boldsymbol{\theta}^t\right), \boldsymbol{\theta}^t-\boldsymbol{\theta}^*\right\rangle \mid \boldsymbol{\theta}^0\right] \leq C_7\left(\log (T / \delta)+\omega^2\right) \tau^* \eta_{\max \left\{0, t-\tau^*\right\}}.
$$
\end{proof}

Looking back at the definitions of $\lambda_0$ and $\Sigma_\pi$, and the discussion in Section \ref{sect:improve-ana}, we derive Lemmas \ref{lemma:range-space-strong-convex} and \ref{lemma:I4} to estimate $\mathbf{I}_4$.

\begin{lemma}
\label{lemma:range-space-strong-convex}
Let $\lambda_0$ as the minimum nonzero singular value of $\Sigma_\pi$. For any $\boldsymbol{\theta}\in \mathcal{R}(\Sigma_\pi)$, we have 
$$
\boldsymbol{\theta}^\top \Sigma_\pi \boldsymbol{\theta} \geq \lambda_0 \|\boldsymbol{\theta}\|_2^2.
$$
\end{lemma}

\begin{lemma}
\label{lemma:I4}
Under Assumption \ref{as:technical}, we have that

\begin{equation}
\label{eq:I4}
\mathbb{E}_{\mu,\pi,\mathbb{P}}\left[\left<\boldsymbol{\theta}^t-\boldsymbol{\theta}^t_*, \Bar{\mathbf{m}}(\boldsymbol{\theta}^t)-\Bar{\mathbf{m}}(\boldsymbol{\theta}^t_*)\right>\mid\boldsymbol{\theta}^0\right] \geq (1-\gamma)\lambda_0\cdot \|\boldsymbol{\theta}^t-\boldsymbol{\theta}^t_*\|^2.
\end{equation}
\end{lemma}

\begin{proof}
Define $d\sim \mu \times \pi$.
To begin with, the Bellman operator $\mathcal{T}^\pi$ is a $\gamma$-contraction with $\ell_2$-norm since $d$ is the stationary distribution of $(s,a)$ corresponding to the policy $\pi$. In details, consider 
\begin{equation}
\label{eq:contraction}
\begin{aligned}
\mathbb{E}_{(s,a)\sim d}\left[\left(\mathcal{T}^\pi Q_1(\boldsymbol{x})-\mathcal{T}^\pi Q_2(\boldsymbol{x})\right)^2\right]&= \gamma^2 \mathbb{E}_{(s,a)\sim d}\left[\mathbb{E}\left[\left( Q_1(\boldsymbol{x}')-Q_2(\boldsymbol{x}')\right)^2\mid s'\sim \mathbb{P}(\cdot| s,a), a'\sim\pi(\cdot|s')\right]\right] \\
&\overset{(i)}{\leq} \gamma^2\mathbb{E}_{(s,a)\sim d}\left[\left( Q_1(\boldsymbol{x})- Q_2(\boldsymbol{x})\right)^2\right],
\end{aligned}
\end{equation}
where (i) follows that $\boldsymbol{x}$ and $\boldsymbol{x}'$ have the same stationary distribution. To simplify the notation, we denote $\mathbb{E}[\cdot]$ as $\mathbb{E}_{\mu,\pi,\mathbb{P}}[\cdot]$ in the proof of this lemma. Then we compute
\begin{eqnarray}
\label{eq:lemma-I4-decompose}
&\ &\mathbb{E}\left[\left<\boldsymbol{\theta}^t-\boldsymbol{\theta}^t_*, \Bar{\mathbf{m}}(\boldsymbol{\theta}^t)-\Bar{\mathbf{m}}(\boldsymbol{\theta}^t_*)\right>\mid\boldsymbol{\theta}^0\right] \nonumber \\
&=& \mathbb{E}\left[\left(\widehat{\Delta}(\boldsymbol{x}, \boldsymbol{x}';\boldsymbol{\theta}^t)-\widehat{\Delta}(\boldsymbol{x}, \boldsymbol{x}';\boldsymbol{\theta}^t_*)\right)\left<\nabla Q(\boldsymbol{x};\boldsymbol{\theta}^0), \boldsymbol{\theta}^t-\boldsymbol{\theta}^t_*\right>\mid\boldsymbol{\theta}^0\right]\nonumber \\
&=& \mathbb{E}\left[\left(\widehat{Q}(\boldsymbol{x};\boldsymbol{\theta}^t)-\widehat{Q}(\boldsymbol{x};\boldsymbol{\theta}^t_*)\right)\left<\nabla Q(\boldsymbol{x};\boldsymbol{\theta}^0), \boldsymbol{\theta}^t-\boldsymbol{\theta}^t_*\right>\mid\boldsymbol{\theta}^0\right]\nonumber \\
&\ & -\gamma  \mathbb{E}\left[\left(\widehat{Q}(\boldsymbol{x}';\boldsymbol{\theta}^t)-\widehat{Q}(\boldsymbol{x}';\boldsymbol{\theta}^t_*)\right)\left<\nabla Q(\boldsymbol{x};\boldsymbol{\theta}^0), \boldsymbol{\theta}^t-\boldsymbol{\theta}^t_*\right>\mid\boldsymbol{\theta}^0\right]\nonumber \\
&=&\mathbb{E} \left[\left(\widehat{Q}(\boldsymbol{x},\boldsymbol{\theta}^t)-\widehat{Q}(\boldsymbol{x},\boldsymbol{\theta}^t_*)\right)^2\right]-\gamma \mathbb{E} \left[\left(\widehat{Q}(\boldsymbol{x},\boldsymbol{\theta}^t)-\widehat{Q}(\boldsymbol{x},\boldsymbol{\theta}^t_*)\right)\cdot \left(\widehat{Q}(\boldsymbol{x}';\boldsymbol{\theta}^t)-\widehat{Q}(\boldsymbol{x}';\boldsymbol{\theta}^t_*)\right)\right] \nonumber \\
&\overset{(i)}{\geq}& \mathbb{E} \left[\left(\widehat{Q}(\boldsymbol{x},\boldsymbol{\theta}^t)-\widehat{Q}(\boldsymbol{x},\boldsymbol{\theta}^t_*)^2\right)\right]-\gamma \sqrt{\mathbb{E} \left[\left(\widehat{Q}(\boldsymbol{x},\boldsymbol{\theta}^t)-\widehat{Q}(\boldsymbol{x},\boldsymbol{\theta}^t_*)\right)^2\right]} \cdot \sqrt{\mathbb{E} \left[\left(\widehat{Q}(\boldsymbol{x}';\boldsymbol{\theta}^t)-\widehat{Q}(\boldsymbol{x}';\boldsymbol{\theta}^t_*)\right)^2\right]}\nonumber  \\
&\overset{(i)}{\geq}& \left(1-\gamma\right) \mathbb{E} \left[\widehat{Q}(\boldsymbol{x},\boldsymbol{\theta}^t)-\widehat{Q}(\boldsymbol{x},\boldsymbol{\theta}^t_*)^2\right]\nonumber \\
&=& (1-\gamma)  (\boldsymbol{\theta}^t-\boldsymbol{\theta}^t_*)^\top \Sigma_\pi (\boldsymbol{\theta}^t-\boldsymbol{\theta}^t_*)\nonumber \\
&\overset{(ii)}{\geq}& (1-\gamma)\lambda_0\cdot \|\boldsymbol{\theta}^t-\boldsymbol{\theta}^t_*\|^2,
\end{eqnarray}
where (i) follows the Cauchy-Schwarz inequality,  (ii) follows \eqref{eq:contraction}, and (iii) follows $\boldsymbol{\theta}^t-\boldsymbol{\theta}^t_*=\boldsymbol{\theta}^t_{\parallel}-\bar{\boldsymbol{\theta}}_{\parallel}\in \mathcal{R}(\Sigma_\pi)$ and Lemma \ref{lemma:range-space-strong-convex}, which provides $\lambda_0$-strong convexity. Thus we complete the proof of Lemma \ref{lemma:I4}.
\end{proof}

Given $\boldsymbol{\theta}^0$, taking the expectation on both sides of \eqref{eq:err-decompose} and plugging \eqref{eq:I1}$\sim$\eqref{eq:I4} into \eqref{eq:err-decompose} yields that
$$
\begin{aligned}
\mathbb{E}_{\mu,\pi,\mathbb{P}}\left[\|\boldsymbol{\theta}^{t+1}-\boldsymbol{\theta}^{t+1}_*\|^2\right.&\left.\mid \boldsymbol{\theta}^0 \right] \leq (1-2\eta_t (1-\gamma)\lambda_0 ) \mathbb{E}\left[\|\boldsymbol{\theta}^{t}-\boldsymbol{\theta}^{t}_*\|^2\mid \boldsymbol{\theta}^0 \right]+C_2\eta_t^2 \\
&+2\eta_tC_5m^{-1/2}\sqrt{\log (T / \delta)}+2\eta_t C_7\left(\log (T / \delta)+C_1^2\right) \tau^* \eta_{\max \left\{0, t-\tau^*\right\}}. 
\end{aligned}
$$
We choose $\eta_t=\frac{1}{2(1-\gamma)\lambda_0 (t+1)}$ and have that
\begin{equation}
\label{eq:err-sum}
\begin{aligned}
(1-\gamma)\lambda_0 (t+1) \mathbb{E}_{\mu,\pi,\mathbb{P}}\left[\|\boldsymbol{\theta}^{t+1}\right.&-\left.\boldsymbol{\theta}^{t+1}_*\|^2\mid \boldsymbol{\theta}^0 \right] \leq (1-\gamma)\lambda_0 t\ \mathbb{E}\left[\|\boldsymbol{\theta}^{t}-\boldsymbol{\theta}^{t}_*\|^2\mid \boldsymbol{\theta}^0 \right]+C_2\eta_t \nonumber \\
&+C_5m^{-1/2}\sqrt{\log (T / \delta)}+ C_7\left(\log (T / \delta)+C_1^2\right) \tau^* \eta_{\max \left\{0, t-\tau^*\right\}}.
\end{aligned}
\end{equation}
Summing \eqref{eq:err-sum} from $t=0,1,\cdots,T-1$ yields that
\begin{eqnarray*}
\mathbb{E}_{\mu,\pi,\mathbb{P}}\left[\|\boldsymbol{\theta}^T-\boldsymbol{\theta}^T_*\|^2\mid \boldsymbol{\theta}^0 \right] &\leq&\frac{1}{(1-\gamma)\lambda_0 T }\sum_{t=0}^{T-1} \left(C_2\eta_t +C_5m^{-1/2}\sqrt{\log (T / \delta)}\right. \\
&\ &\left.+ C_7\left(\log (T / \delta)+C_1^2\right) \tau^* \eta_{\max \left\{0, t-\tau^*\right\}}\right) \\
&\leq& \frac{C_2(\log T+1)}{2(1-\gamma)^2\lambda_0^2 T}+ \frac{C_5m^{-1/2}\sqrt{\log (T / \delta)}}{(1-\gamma)\lambda_0}  +\frac{C_8 \tau^*\left(\log (T / \delta)+1\right) \log T}{2(1-\gamma)^2\lambda_0^2 T}.
\end{eqnarray*}
Therefore, according to the gradient bound \eqref{eq:I1}, we have 
\begin{eqnarray*}
\mathbb{E}_{\mu,\pi}\left[\left(\widehat{Q}(\boldsymbol{x};\boldsymbol{\theta}^T)-\widehat{Q}(\boldsymbol{x};\boldsymbol{\theta}^*)\right)^2\mid \boldsymbol{\theta}^0 \right] &=& \mathbb{E}\left[\left(\widehat{Q}(\boldsymbol{x};\boldsymbol{\theta}^T)-\widehat{Q}(\boldsymbol{x};\boldsymbol{\theta}^T_*)\right)^2\mid \boldsymbol{\theta}^0 \right] \\
&\leq& C_3^2 m \mathbb{E}\left[\|\boldsymbol{\theta}^T-\boldsymbol{\theta}^T_*\|^2\mid \boldsymbol{\theta}^0 \right] \\
& \leq& \frac{C_2^3(\log T+1)}{2(1-\gamma)^2\lambda_0^2 T}+ \frac{C_2^2C_5m^{-1/2}\sqrt{\log (T / \delta)}}{(1-\gamma)\lambda_0} \\
&\ &+\frac{C_2^2C_8 \tau^*\left(\log (T / \delta)+1\right) \log T}{2(1-\gamma)^2\lambda_0^2 T}
\end{eqnarray*}
with probability at least $1-2\delta-2L\exp{(-C_6 m)}$.
Let $\widetilde{C}_1=\max\{1, C_1\}, \widetilde{C}_2=C_6, \widetilde{C}_3=\frac{C_2^3}{2}, \widetilde{C}_4=\frac{C_2^2C_5}{2}$, and $\widetilde{C}_5=\frac{C_2^2C_8}{2}$, and we complete the proof.
\end{proof}

\subsection{Proof of Theorem \ref{theorem:total}}
\begin{proof}
Let $(s,a)\sim \mu\times\pi =:d$. To simplify the notation, we denote $\mathbb{E}[\cdot]$ as $\mathbb{E}_{(s,a)\sim d}[\cdot]$ in this subsection.
Note that
\begin{equation}
\label{eq:total-optimal}
\begin{aligned}
\mathbb{E}\left[\left(Q(\boldsymbol{x};\boldsymbol{\theta}^T)\right.\right.&-\left.\left.Q^*(s,a)\right)^2\mid \boldsymbol{\theta}^0 \right]\leq3\mathbb{E}\left[\left(Q(\boldsymbol{x};\boldsymbol{\theta}^T)-\widehat{Q}(\boldsymbol{x};\boldsymbol{\theta}^T)\right)^2\mid \boldsymbol{\theta}^0 \right] \nonumber \\
&+3\mathbb{E}\left[\left(\widehat{Q}(\boldsymbol{x};\boldsymbol{\theta}^T)-\widehat{Q}(\boldsymbol{x};\boldsymbol{\theta}^*)\right)^2\mid \boldsymbol{\theta}^0 \right]+3\mathbb{E}\left[\left(\widehat{Q}(\boldsymbol{x};\boldsymbol{\theta}^*)-Q^*(s,a)\right)^2\mid \boldsymbol{\theta}^0 \right].
\end{aligned}
\end{equation}
By Lemma \ref{lemma:hessian}, we have
\begin{equation}
\label{eq:gap-q-qhat}
\mathbb{E}\left[\left(Q(\boldsymbol{x};\boldsymbol{\theta}^T)-\widehat{Q}(\boldsymbol{x};\boldsymbol{\theta}^T)\right)^2\mid \boldsymbol{\theta}^0 \right]\leq C_8 m^{-1}
\end{equation}
with probability at least $1-\delta$. Recall that $\widehat{Q}(\boldsymbol{x};\boldsymbol{\theta}^*)$ is the fixed point of $\Pi_{\mathcal{F}_{\omega, m}}\mathcal{T}$ and $Q^*(s,a)$ is the fixed point of $\mathcal{T}$. We define the $\ell_2$-norm $\|f(s,a)\|_d^2=\mathbb{E}_{(s,a)\sim d}\left[f(s,a)^2\right]$. Thus
\begin{eqnarray*}
\left\|\widehat{Q}(\boldsymbol{x};\boldsymbol{\theta}^*)-Q^*(s,a)\right\|_d&= &\left\|\widehat{Q}(\boldsymbol{x};\boldsymbol{\theta}^*)-\Pi_{\mathcal{F}_{\omega, m}}Q^*(s,a)+\Pi_{\mathcal{F}_{\omega, m}}Q^*(s,a)-Q^*(s,a) \right\|_d \\
&\overset{(i)}{=}&\left\|\Pi_{\mathcal{F}_{\omega, m}}\mathcal{T}\widehat{Q}(\boldsymbol{x};\boldsymbol{\theta}^*)-\Pi_{\mathcal{F}_{\omega, m}}\mathcal{T}Q^*(s,a)+\Pi_{\mathcal{F}_{\omega, m}}Q^*(s,a)-Q^*(s,a) \right\|_d \\
&\leq&\left\|\Pi_{\mathcal{F}_{\omega, m}}\mathcal{T}\widehat{Q}(\boldsymbol{x};\boldsymbol{\theta}^*)-\Pi_{\mathcal{F}_{\omega, m}}\mathcal{T}Q^*(s,a)\right\|_d+\left\|\Pi_{\mathcal{F}_{\omega, m}}Q^*(s,a)-Q^*(s,a) \right\|_d \\
&\overset{(ii)}{\leq}&\gamma \left\|\widehat{Q}(\boldsymbol{x};\boldsymbol{\theta}^*)-Q^*(s,a)\right\|_d+\left\|\Pi_{\mathcal{F}_{\omega, m}}Q^*(s,a)-Q^*(s,a) \right\|_d,
\end{eqnarray*}
where (i) is due to the properties of the fixed point, and (ii) is due to $\Pi_{\mathcal{F}_{\omega, m}}\mathcal{T}$ is $\gamma$-contractive on the $\infty$-norm. This further means that
\begin{equation}
\label{eq:prop-fix-point}
\left\|\widehat{Q}(\boldsymbol{x};\boldsymbol{\theta}^*)-Q^*(s,a)\right\|_d^2\leq \frac{1}{(1-\gamma)^2}\left\|\Pi_{\mathcal{F}_{\omega, m}}Q^*(s,a)-Q^*(s,a) \right\|_d^2.
\end{equation}
Plugging \eqref{eq:gap-q-qhat} and \eqref{eq:prop-fix-point} into \eqref{eq:total-optimal} and using Theorem \ref{theorem:pi-f}, we complete the proof.
\end{proof}

\section{Convergence Results of Neural Q-learning}
\label{sect:conv-ql}
\subsection{Neural Q-Learning Algorithm}
For neural Q-learning, let us redefine some of the above notations. Let the optimal Q-function be $Q^*(s,a)=\sup_\pi Q^\pi(s,a)$ for all state action pairs $(s,a)$, then the optimal sequence of actions that maximizes the expected cumulative reward will follow $a_t=\mathop{\mathrm{argmax}}_{a'\in \mathcal{A}} Q^*(s_t,a'), t\geq0$. Therefore, to obtain a near-optimal policy, it is sufficient to find some $\hat{Q}$ that approximates $Q^*$ well. Define the Bellman optimality operator $\mathcal{T}$ as
$$\mathcal{T} Q(s,a):=r(s, a)+\gamma \mathbb{E}\left[\max_{a'} Q(s', a')\mid s'\sim \mathbb{P}(\cdot\mid s,a)\right], $$
for any $(s,a)$.
Let us remain the definition of the local linearization function class $\mathcal{F}_{\omega,m}$ introduced in \eqref{defn:local-linearization}. Consider the 
MSPBE minimization problem with multi-layer neural network approximation:
\begin{equation*}
\min_{\boldsymbol{\theta}\in S_\omega} \mathbb{E}_{\mu,\pi,\mathbb{P}} \left[\left(Q(\boldsymbol{x};\boldsymbol{\theta})- \Pi_{\mathcal{F}_{\omega,m}}\mathcal{T}Q(\boldsymbol{x};\boldsymbol{\theta})\right)^2\right].
\end{equation*}
Then the projected neural Q-learning algorithm can be written as follows:
\begin{equation}
\label{eq:semi-grad-ql} 
\boldsymbol{\theta}^{t+1}=\Pi_{S_\omega}\Big(\boldsymbol{\theta}^t-\eta_t \boldsymbol{g}\left(\boldsymbol{\theta}^t\right)\!\!\Big),\quad\mbox{with}\quad\boldsymbol{g}(\boldsymbol{\theta}^t)= \Delta\left(s_t, a_t, s_{t+1} ; \boldsymbol{\theta}^t\right)\cdot \nabla_{\boldsymbol{\theta}}  Q(\phi(s_t,a_t);\boldsymbol{\theta}^t)
\end{equation}
where
\begin{align}
\label{eq:td-error-ql}
\Delta\left(s, a, s^{\prime} ; \boldsymbol{\theta}^t\right)=&Q(\phi(s,a);\boldsymbol{\theta}^t)-\Big(r(s, a)+\gamma \max _{b \in \mathcal{A}} Q\left(\phi\left(s^{\prime}, b\right);\boldsymbol{\theta}^t\right)\Big).
\end{align}
The algorithm details can be described by Algorithm \ref{alg:Neural-ql} as follows.
\begin{algorithm}[H]
\caption{Neural Q-Learning with Markovian Sampling}
\label{alg:Neural-ql}
\begin{algorithmic}
\STATE \textbf{Input:} A learning policy $\pi$, a discount factor $\gamma\in(0,1)$, a sequence of learning rates $\{\eta_t\}_{t\geq0}$, a maximum iteration number $T$, a projection radius $\omega>0$, a Q network with architecture \eqref{eq:nn}.\vspace{0.05cm}
\STATE \textbf{Initialization:} Generate each entry of $\boldsymbol{W}_l^0$ independently from $\mathcal{N}(0, 1)$, for $l=1,2,\cdots, L$, and each entry of $\boldsymbol{b}$ independently from $\text{Unif} \{-1,+1\}$. 
\vspace{0.05cm}
\FOR{$t=0,1,\cdots,T-1$}
\STATE Sample $(s_t,a_t,r_t,s_{t+1})$ from the learning policy $\pi$ with $a_t\sim\pi(\cdot|s_t)$.
\STATE Compute the TD error $\Delta_t$ by \eqref{eq:td-error-ql}.
\STATE Update $\boldsymbol{\theta}^{t+1}$ by the projected stochastic semi-gradient step \eqref{eq:semi-grad-ql}.
\ENDFOR\vspace{0.05cm}
\STATE \textbf{Output:} $\boldsymbol{\theta}^T$.
\end{algorithmic}
\end{algorithm}

\subsection{Global Convergence}
\label{sect:proof-ql}

Similar to Section \ref{section:conv}, we define the function class $\mathcal{F}_{\omega, m}$ as a collection of all local linearization of $Q(\boldsymbol{x};\boldsymbol{\theta})$ at the initial point $\boldsymbol{\theta}^0$:
$$
\mathcal{F}_{\omega, m}:=\left\{\widehat{Q}(\boldsymbol{x};\boldsymbol{\theta})=Q(\boldsymbol{x};\boldsymbol{\theta}^0)+\left<\nabla_{\boldsymbol{\theta}} Q(\boldsymbol{x};\boldsymbol{\theta}^0), \boldsymbol{\theta}-\boldsymbol{\theta}^0\right>,\ \boldsymbol{\theta} \in S_\omega\right\}.
$$
Let $\widehat{Q}(\cdot\,;\boldsymbol{\theta}^*)\in\mathcal{F}_{\omega,m}$, and $\widehat{\Delta}\left(s, a, s^{\prime} ; \boldsymbol{\theta}\right)$ has the same structure as $\Delta\left(s, a, s^{\prime} ; \boldsymbol{\theta}\right)$ expect that the function $Q(\cdot;\boldsymbol{\theta})$ is replaced by $\widehat{Q}(\cdot;\boldsymbol{\theta})$. The stationary point $\boldsymbol{\theta}^*$ satisfies $\widehat{Q}(\boldsymbol{x};\boldsymbol{\theta}^*)=\Pi_{\mathcal{F}_{\omega,m}}\mathcal{T}\widehat{Q}(\boldsymbol{x};\boldsymbol{\theta}^*)$ for neural Q-learning. We redefine $\Xi_\beta$ by replacing the Bellman operator $\mathcal{T}^\pi$ in Section \ref{section:conv} with the Bellman optimality operator $\mathcal{T}$. A point $\boldsymbol{\theta}^*\in \Xi_\omega$ if and only if 
$$
\mathbb{E}_{\mu,\pi,\mathbb{P}}\left[\widehat{\Delta}\left(s, a, s^{\prime} ; \boldsymbol{\theta}^*\right)\big\langle\nabla_{\boldsymbol{\theta}} \widehat{Q}\left(\phi(s,a);\boldsymbol{\theta}^* \right), \boldsymbol{\theta}-\boldsymbol{\theta}^*\big\rangle\right] \geq 0.
$$

The maximum operator introduced by the Bellman optimality operator significantly sophisticates the analysis.  
Let us remain the definition of $\Sigma_\pi$ in \eqref{eq:feature-matrix}, and we define $\Sigma^*_\pi(\boldsymbol{\theta})$ as follows:
\begin{eqnarray}
\label{eq:feature-matrix-ql} 
\mathbb{E}_{\mu, \pi}\left[\nabla_{\boldsymbol{\theta}}Q(\phi(s,a^{\boldsymbol{\theta}}_{\max});\boldsymbol{\theta}^0)\nabla_{\boldsymbol{\theta}}Q(\phi(s,a^{\boldsymbol{\theta}}_{\max});\boldsymbol{\theta}^0)^\top\right],
\end{eqnarray}
where $a^{\boldsymbol{\theta}}_{\max}=\arg\max_{a\in\mathcal{A}}\left|\left<\nabla_{\boldsymbol{\theta}}Q(s,a;\boldsymbol{\theta}^0), \boldsymbol{\theta}\right>\right|$. To facilitate the analysis of neural Q-learning,  we further assume the following regularity condition introduced by \cite{xu2020finite}.
\begin{assumption}
\label{as:regularity}
$\exists\nu\in(0,1)$ such that $(1-\nu)^2\Sigma_\pi-\gamma^2\Sigma_\pi^*(\boldsymbol{\theta})\succeq 0$ for any $\boldsymbol{\theta}^0$ and $\boldsymbol{\theta}\in S_\omega$.
\end{assumption}

The original version of this assumption comes from \cite{xu2020finite}, which requires a strict positive definite condition: $(1-\nu)^2\Sigma_\pi-\gamma^2\Sigma_\pi^*(\boldsymbol{\theta})\succ 0$. 
Under this additional assumption, \cite{xu2020finite} obtained an $\tilde{\mathcal{O}}(\epsilon^{-2})$ sample complexity for neural Q-learning. A similar complexity result was also derived in \cite{cai2023neural} under a similar regularity condition on the learning policy $\pi$. At this time, we relax it to the positive semi-definiteness ($\succeq0$) and provide a convergence result of neural Q-learning. See Theorem \ref{theorem:pi-f-ql}.

\begin{theorem}
\label{theorem:pi-f-ql}
Suppose Assumptions \ref{as:lamda}, \ref{as:markov} and \ref{as:regularity} hold. We set $\omega=\widetilde{C}_1 $ and the learning rate $\eta_t=\frac{1}{2\nu\lambda_0 (t+1)}$. If the feature map $\|\phi(s,a)\|=1$ for each state-action pair $(s,a)$ and the network width $m \geq m^*$, then the output $\boldsymbol{\theta}^T$ of neural Q-learning algorithm (i.e. \eqref{eq:semi-grad-ql}) satisfies
\begin{eqnarray*}
&\ &\mathbb{E}\left[\big(\widehat{Q}(\boldsymbol{x};\boldsymbol{\theta}^T)-\widehat{Q}(\boldsymbol{x};\boldsymbol{\theta}^*)\big)^2\mid \boldsymbol{\theta}^0 \right]\leq \frac{\widetilde{C}_3(\log T+1)}{\nu^2\lambda_0^2T}+ \frac{\widetilde{C}_4m^{-1/2}}{\nu\lambda_0}\cdot\sqrt{ \log (T / \delta)} +\frac{\widetilde{C}_5\tau^* \left(\log (T / \delta)+1\right) \log T}{ \nu^2\lambda_0^2 T},
\end{eqnarray*}
with probability at least $1-2\delta-2L\exp\!\big(-\widetilde{C}_2 m\big)$, where $\tau^*$ is the mixing time of Markov chain in Assumption \ref{as:markov}, and $\widetilde{C}_1,\cdots,\widetilde{C}_5>0$ are universal constants.
\end{theorem}

\begin{proof}
For a little notation abuse, we redefine
\begin{eqnarray*}
\mathbf{g}(\boldsymbol{\theta}^t) &=& \Delta(s_t,a_t,s'_t,\boldsymbol{\theta}^t)\cdot  \nabla_{\boldsymbol{\theta}}  Q(\boldsymbol{x}_t;\boldsymbol{\theta}^t), \quad
\Bar{\mathbf{g}}(\boldsymbol{\theta}^t) \ = \ \mathbb{E}_{\mu, \pi,\mathbb{P}}\left[\mathbf{g}(\boldsymbol{\theta}^t)\right]\\ 
\mathbf{m}(\boldsymbol{\theta}^t) &=& \widehat{\Delta} (s_t,a_t,s'_t,\boldsymbol{\theta}^t)\cdot\nabla_{\boldsymbol{\theta}}  Q(\boldsymbol{x}_t;\boldsymbol{\theta}^0),
\quad \Bar{\mathbf{m}}(\boldsymbol{\theta}^t)\ =\ \mathbb{E}_{\mu, \pi,\mathbb{P}}\left[\mathbf{m}(\boldsymbol{\theta}^t)\right],
\end{eqnarray*}
where 
\begin{eqnarray*}
\Delta(s_t,a_t,s'_t,\boldsymbol{\theta}^t)&=& Q(\boldsymbol{x}_t;\boldsymbol{\theta}^t)-\left(r(s_t, a_t)+\gamma \max_{b\in\mathcal{A}} Q(\phi(s_{t+1},b);\boldsymbol{\theta}^t)\right), \\
\widehat{\Delta} (s_t,a_t,s'_t,\boldsymbol{\theta}^t) &=&\widehat{Q}(\boldsymbol{x}_t;\boldsymbol{\theta}^t)-\left(r(s_t, a_t)+\gamma \max_{b\in\mathcal{A}} \widehat{Q}(\phi(s_{t+1},b);\boldsymbol{\theta}^t)\right).
\end{eqnarray*}
Let $\Delta_t=\Delta(s_t,a_t,s'_t,\boldsymbol{\theta}^t)$ and $\widehat{\Delta}_t=\widehat{\Delta} (s_t,a_t,s'_t,\boldsymbol{\theta}^t)$. Similarly, \eqref{eq:err-decompose} can be derived in neural Q-learning. To estimate the terms $\mathbf{I}_1\sim\mathbf{I}_4$, we can apply Lemmas \ref{lemma:I1-and-I2} and \ref{lemma:I3}. However, due to the utilization of the Bellman optimality operator in neural Q-learning, some modifications based on Lemma \ref{lemma:I4} are required.
\begin{lemma}
\label{lemma:I4-ql}
Under Assumption \ref{as:regularity}, we have that

\begin{equation}
\label{eq:I4-ql}
\mathbb{E}_{\mu,\pi,\mathbb{P}}\left[\left<\boldsymbol{\theta}^t-\boldsymbol{\theta}^t_*, \Bar{\mathbf{m}}(\boldsymbol{\theta}^t)-\Bar{\mathbf{m}}(\boldsymbol{\theta}^t_*)\right>\mid\boldsymbol{\theta}^0\right] \geq \nu\lambda_0\cdot \|\boldsymbol{\theta}^t-\boldsymbol{\theta}^t_*\|^2.
\end{equation}
\end{lemma}

\begin{proof}
To simplify the notation, we denote $\mathbb{E}[\cdot]$ as $\mathbb{E}_{\mu,\pi,\mathbb{P}}[\cdot]$ in the proof of this lemma.
Define $\widehat{Q}^{\#}(s;\boldsymbol{\theta}):=\max_{a\in\mathcal{A}}\widehat{Q}(\phi(s,a);\boldsymbol{\theta})$. Then we have
\begin{eqnarray}
\label{eq:lemma-I4-decompose-ql}
&\ &\mathbb{E}\left[\left<\boldsymbol{\theta}^t-\boldsymbol{\theta}^t_*, \Bar{\mathbf{m}}(\boldsymbol{\theta}^t)-\Bar{\mathbf{m}}(\boldsymbol{\theta}^t_*)\right>\mid\boldsymbol{\theta}^0\right] \nonumber \\
&=& \mathbb{E}\left[\left(\widehat{\Delta}(s_t,a_t,s;_t,\boldsymbol{\theta}^t)-\widehat{\Delta}(s_t,a_t,s;_t,\boldsymbol{\theta}^t_*)\right)\left<\nabla Q(\boldsymbol{x};\boldsymbol{\theta}^0), \boldsymbol{\theta}^t-\boldsymbol{\theta}^t_*\right>\mid\boldsymbol{\theta}^0\right]\nonumber \\
&=& \mathbb{E}\left[\left(\widehat{Q}(\boldsymbol{x};\boldsymbol{\theta}^t)-\widehat{Q}(\boldsymbol{x};\boldsymbol{\theta}^t_*)\right)\left<\nabla Q(\boldsymbol{x};\boldsymbol{\theta}^0), \boldsymbol{\theta}^t-\boldsymbol{\theta}^t_*\right>\mid\boldsymbol{\theta}^0\right]\nonumber \\
&\ & -\gamma  \mathbb{E}\left[\left(\widehat{Q}^{\#}(s;\boldsymbol{\theta}^t)-\widehat{Q}^{\#}(s;\boldsymbol{\theta}^t_*)\right)\left<\nabla Q(\boldsymbol{x};\boldsymbol{\theta}^0), \boldsymbol{\theta}^t-\boldsymbol{\theta}^t_*\right>\mid\boldsymbol{\theta}^0\right] \\
&=&\mathbb{E} \left[\left(\widehat{Q}(\boldsymbol{x},\boldsymbol{\theta}^t)-\widehat{Q}(\boldsymbol{x},\boldsymbol{\theta}^t_*)\right)^2\right]-\gamma \mathbb{E} \left[\left(\widehat{Q}(\boldsymbol{x},\boldsymbol{\theta}^t)-\widehat{Q}(\boldsymbol{x},\boldsymbol{\theta}^t_*)\right)\cdot \left(\widehat{Q}^{\#}(s;\boldsymbol{\theta}^t)-\widehat{Q}^{\#}(s;\boldsymbol{\theta}^t_*)\right)\right]. \nonumber
\end{eqnarray}
For the second term of \eqref{eq:lemma-I4-decompose-ql}, we consider
\begin{eqnarray}
\label{eq:regularity-ql}
\mathbb{E} \left[\left(\widehat{Q}^{\#}(s;\boldsymbol{\theta}^t)-\widehat{Q}^{\#}(s;\boldsymbol{\theta}_*^t)\right)^2\right]&\leq & \mathbb{E} \left[\max_{a\in\mathcal{A}} \left|\widehat{Q}(s,a;\boldsymbol{\theta}^t)-\widehat{Q}(s,a;\boldsymbol{\theta}^t_*) \right|^2 \right] \nonumber\\
&\overset{(i)}{=}& \mathbb{E} \left[\max_{a\in\mathcal{A}}\left| \widehat{Q}(s,a;\boldsymbol{\theta}^t-\boldsymbol{\theta}^t_*)\right|^2 \right] \nonumber\\
&=&(\boldsymbol{\theta}^t-\boldsymbol{\theta}^t_*)^\top \Sigma_\pi^*(\boldsymbol{\theta}^t-\boldsymbol{\theta}^t_*)\nonumber \\
&\overset{(ii)}{\leq} & \frac{(1-\nu)^2}{\gamma^2} (\boldsymbol{\theta}^t-\boldsymbol{\theta}^t_*)^\top \Sigma_\pi\cdot (\boldsymbol{\theta}^t-\boldsymbol{\theta}^t_*)\nonumber \\
&\overset{(iii)}{=}&\frac{(1-\nu)^2}{\gamma^2} \mathbb{E} \left[\left(\widehat{Q}(\boldsymbol{x},\boldsymbol{\theta}^t)-\widehat{Q}(\boldsymbol{x},\boldsymbol{\theta}^t_*)\right)^2\right],
\end{eqnarray}
where (i) and (iii) follow that $\widehat{Q}(\boldsymbol{x};\cdot)$ is linear, and (ii) follows Assumption \ref{as:regularity}. Therefore, 
\begin{eqnarray*}
&\ &\mathbb{E}\left[\left<\boldsymbol{\theta}^t-\boldsymbol{\theta}^t_*, \Bar{\mathbf{m}}(\boldsymbol{\theta}^t)-\Bar{\mathbf{m}}(\boldsymbol{\theta}^t_*)\right>\mid\boldsymbol{\theta}^0\right] \\
&=&\mathbb{E} \left[\left(\widehat{Q}(\boldsymbol{x},\boldsymbol{\theta}^t)-\widehat{Q}(\boldsymbol{x},\boldsymbol{\theta}^t_*)\right)^2\right]-\gamma \mathbb{E} \left[\left(\widehat{Q}(\boldsymbol{x},\boldsymbol{\theta}^t)-\widehat{Q}(\boldsymbol{x},\boldsymbol{\theta}^t_*)\right)\cdot \left(\widehat{Q}^{\#}(s;\boldsymbol{\theta}^t)-\widehat{Q}^{\#}(s;\boldsymbol{\theta}^t_*)\right)\right] \\
&\geq& \mathbb{E} \left[\left(\widehat{Q}(\boldsymbol{x},\boldsymbol{\theta}^t)-\widehat{Q}(\boldsymbol{x},\boldsymbol{\theta}^t_*)^2\right)\right]-\gamma \sqrt{\mathbb{E} \left[\left(\widehat{Q}(\boldsymbol{x},\boldsymbol{\theta}^t)-\widehat{Q}(\boldsymbol{x},\boldsymbol{\theta}^t_*)\right)^2\right]} \cdot \sqrt{\mathbb{E} \left[\left(\widehat{Q}^{\#}(s;\boldsymbol{\theta}^t)-\widehat{Q}^{\#}(s;\boldsymbol{\theta}^t_*)\right)^2\right]}  \\
&\overset{(i)}{\geq}& \left(1-\gamma\cdot\frac{1-\nu}{\gamma}\right) \mathbb{E} \left[\widehat{Q}(\boldsymbol{x},\boldsymbol{\theta}^t)-\widehat{Q}(\boldsymbol{x},\boldsymbol{\theta}^t_*)^2\right] \\
&=& \nu\cdot  (\boldsymbol{\theta}^t-\boldsymbol{\theta}^t_*)^\top \Sigma_\pi (\boldsymbol{\theta}^t-\boldsymbol{\theta}^t_*) \\
&\overset{(ii)}{\geq}& \nu\lambda_0\cdot \|\boldsymbol{\theta}^t-\boldsymbol{\theta}^t_*\|^2,
\end{eqnarray*}
where (i) follows \eqref{eq:regularity-ql}, and (ii) follows $\boldsymbol{\theta}^t-\boldsymbol{\theta}^t_*=\boldsymbol{\theta}^t_{\parallel}-\boldsymbol{\theta}^*\in \mathcal{R}(\mathbf{\Sigma_\pi})$ and Lemma \ref{lemma:range-space-strong-convex}, which provides $\lambda_0$-strong convexity. 
\end{proof}

Now given $\boldsymbol{\theta}^0$, we can deduce that 
$$
\begin{aligned}
\mathbb{E}_{\mu,\pi,\mathbb{P}}\left[\|\boldsymbol{\theta}^{t+1}\right.&-\left.\boldsymbol{\theta}^{t+1}_*\|^2\mid \boldsymbol{\theta}^0 \right] \leq (1-2\eta_t \nu\lambda_0 ) \mathbb{E}\left[\|\boldsymbol{\theta}^{t}-\boldsymbol{\theta}^{t}_*\|^2\mid \boldsymbol{\theta}^0 \right]+C_1\eta_t^2 \\
&+2\eta_tC_2m^{-1/2}\sqrt{\log (T / \delta)}+2\eta_t C_3\left(\log (T / \delta)+C_1^2\right) \tau^* \eta_{\max \left\{0, t-\tau^*\right\}},
\end{aligned}
$$
with probability at least $1-2\delta-2L\exp (-C_4 m)$, where $\left\{C_i>0\right\}_{i=1, \ldots, 4}$ are universal constants in this subsection. Choosing $\eta_t=\frac{1}{2\nu\lambda_0 (t+1)}$ can derive the similar results as \eqref{eq:err-sum}. This suggests that we can utilize the techniques outlined in Section \ref{section:appendix-pi-f} to finalize the remaining proof of Theorem \ref{theorem:pi-f-ql}. As a result, we conclude the proof.

\end{proof}

\section{Details of Section \ref{section:minmax-ql-conv}}
\label{section:minimax-conv}
We formally describe the minimax neural Q-learning method in Algorithm \ref{alg:minimax-q-learning}.

\begin{algorithm}[H]
\caption{Minimax Neural Q-Learning with Gaussian Initialization}
\label{alg:minimax-q-learning}
\begin{algorithmic}
\STATE \textbf{Input:} A learning policy pair $\pi=(\pi^1,\pi^2)$, a discount factor $\gamma\in(0,1)$, a sequence of learning rates $\{\eta_t\}_{t\geq0}$, a maximum iteration number $T$, a projection radius $\omega>0$, a Q network with architecture \eqref{eq:nn}.\vspace{0.05cm}
\STATE \textbf{Initialization:} Generate each entry of $\boldsymbol{W}_l^0$ independently from $\mathcal{N}(0, 1)$, for $l=1,2,\cdots, L$, and each entry of $\boldsymbol{b}$ independently from $\text{Unif} \{-1,+1\}$. 
\FOR{$t=0,1,\cdots,T-1$}
\STATE Sample $(s_t,a^1_t,a^2_t,r_t,s_{t+1})$ from the learning policy pair $\pi$ with $a^1_t\sim\pi^1(\cdot|s_t), a^2_t\sim\pi^2(\cdot|s_t)$.
\STATE Compute the TD error $\Delta_t$ by \eqref{eq:td-error-minimax-ql}.
\STATE Update $\boldsymbol{\theta}^{t+1}$ by the projected stochastic semi-gradient step \eqref{eq:semi-grad-minimax-ql}.
\ENDFOR
\STATE \textbf{Output:} $\boldsymbol{\theta}^T$.
\end{algorithmic}
\end{algorithm}






\subsection{Proof of Theorem \ref{theorem:minimax-pi-f}}
The proof of Theorem \ref{theorem:minimax-pi-f} is similar to Sections \ref{section:appendix-pi-f} and \ref{sect:proof-ql}. However, due to the difference in Bellman operators, we still need to make some modifications to Lemma \ref{lemma:I4} or Lemma \ref{lemma:I4-ql}. See Lemma \ref{lemma:I4-minimax-ql}.

\begin{lemma}
\label{lemma:I4-minimax-ql}
Under Assumption \ref{as:minimax-regularity}, we have that
\begin{equation}
\label{eq:I4-minimax-ql}
\mathbb{E}_{\mu,\pi,\mathbb{P}}\left[\left<\boldsymbol{\theta}^t-\boldsymbol{\theta}^t_*, \Bar{\mathbf{m}}(\boldsymbol{\theta}^t)-\Bar{\mathbf{m}}(\boldsymbol{\theta}^t_*)\right>\mid\boldsymbol{\theta}^0\right] \geq \nu\lambda_0\cdot \|\boldsymbol{\theta}^t-\boldsymbol{\theta}^t_*\|^2.
\end{equation}
\end{lemma}

\begin{proof}
To simplify the notation, we denote $\mathbb{E}[\cdot]$ as $\mathbb{E}_{\mu,\pi,\mathbb{P}}[\cdot]$ in the proof of this lemma.
Define $\widehat{Q}^{\#}(s;\boldsymbol{\theta}):=\max_{a^1\in\mathcal{A}}\min_{a^2\in\mathcal{A}}\widehat{Q}(\phi(s,a^1,a^2);\boldsymbol{\theta})$. Define the sets $\mathcal{S}_+ =\{s:\widehat{Q}^{\#}(s;\boldsymbol{\theta}^t)>\widehat{Q}^{\#}(s;\boldsymbol{\theta}^t_*)\}$ and $\mathcal{S}_- = \mathcal{S} / \mathcal{S}_+$. For each $s\in \mathcal{S}_+$, 
\begin{eqnarray*}
\widehat{Q}^{\#}(s;\boldsymbol{\theta}^t)-\widehat{Q}^{\#}(s;\boldsymbol{\theta}^t_*)&=& \Big<\nabla_{\boldsymbol{\theta}} Q \left(\phi(s,a^1_{\boldsymbol{\theta}^t},a^2_{\boldsymbol{\theta}^t});\boldsymbol{\theta}^0\right), \boldsymbol{\theta}^t\Big>-\left<\nabla_{\boldsymbol{\theta}} Q \left(\phi(s,a^1_{\boldsymbol{\theta}^t_*},a^2_{\boldsymbol{\theta}^t_*});\boldsymbol{\theta}^0\right), \boldsymbol{\theta}^t_*\right> \\
&=&\left(\Big<\nabla_{\boldsymbol{\theta}} Q \left(\phi(s,a^1_{\boldsymbol{\theta}^t},a^2_{\boldsymbol{\theta}^t});\boldsymbol{\theta}^0\right), \boldsymbol{\theta}^t\Big>-\left<\nabla_{\boldsymbol{\theta}} Q \left(\phi(s,a^1_{\boldsymbol{\theta}^t},a^2_{\boldsymbol{\theta}^t_*});\boldsymbol{\theta}^0\right), \boldsymbol{\theta}^t\right>\right)- \\
&\ &\left<\nabla_{\boldsymbol{\theta}} Q \left(\phi(s,a^1_{\boldsymbol{\theta}^t},a^2_{\boldsymbol{\theta}^t_*});\boldsymbol{\theta}^0\right), \boldsymbol{\theta}^t-\boldsymbol{\theta}^t_*\right>- \\
&\ &\left(\Big<\nabla_{\boldsymbol{\theta}} Q \left(\phi(s,a^1_{\boldsymbol{\theta}^t},a^2_{\boldsymbol{\theta}^t_*});\boldsymbol{\theta}^0\right), \boldsymbol{\theta}^t_*\Big>-\left<\nabla_{\boldsymbol{\theta}} Q \left(\phi(s,a^1_{\boldsymbol{\theta}^t_*},a^2_{\boldsymbol{\theta}^t_*});\boldsymbol{\theta}^0\right), \boldsymbol{\theta}^t_*\right>\right) \\
&\leq & \left<\nabla_{\boldsymbol{\theta}} Q \left(\phi(s,a^1_{\boldsymbol{\theta}^t},a^2_{\boldsymbol{\theta}^t_*});\boldsymbol{\theta}^0\right), \boldsymbol{\theta}^t-\boldsymbol{\theta}^t_*\right>
\end{eqnarray*}
and 
\begin{eqnarray*}
\widehat{Q}^{\#}(s;\boldsymbol{\theta}^t)-\widehat{Q}^{\#}(s;\boldsymbol{\theta}^t_*)&=&\left(\Big<\nabla_{\boldsymbol{\theta}} Q \left(\phi(s,a^1_{\boldsymbol{\theta}^t},a^2_{\boldsymbol{\theta}^t});\boldsymbol{\theta}^0\right), \boldsymbol{\theta}^t\Big>-\left<\nabla_{\boldsymbol{\theta}} Q \left(\phi(s,a^1_{\boldsymbol{\theta}^t_*},a^2_{\boldsymbol{\theta}^t});\boldsymbol{\theta}^0\right), \boldsymbol{\theta}^t\right>\right)- \\
&\ &\left<\nabla_{\boldsymbol{\theta}} Q \left(\phi(s,a^1_{\boldsymbol{\theta}^t_*},a^2_{\boldsymbol{\theta}^t});\boldsymbol{\theta}^0\right), \boldsymbol{\theta}^t-\boldsymbol{\theta}^t_*\right>- \\
&\ &\left(\Big<\nabla_{\boldsymbol{\theta}} Q \left(\phi(s,a^1_{\boldsymbol{\theta}^t_*},a^2_{\boldsymbol{\theta}^t});\boldsymbol{\theta}^0\right), \boldsymbol{\theta}^t_*\Big>-\left<\nabla_{\boldsymbol{\theta}} Q \left(\phi(s,a^1_{\boldsymbol{\theta}^t_*},a^2_{\boldsymbol{\theta}^t_*});\boldsymbol{\theta}^0\right), \boldsymbol{\theta}^t_*\right>\right) \\
&\geq& \left<\nabla_{\boldsymbol{\theta}} Q \left(\phi(s,a^1_{\boldsymbol{\theta}^t_*},a^2_{\boldsymbol{\theta}^t});\boldsymbol{\theta}^0\right), \boldsymbol{\theta}^t-\boldsymbol{\theta}^t_*\right>.
\end{eqnarray*}
In the same way, for each $s\in \mathcal{S}_-$,
\begin{eqnarray*}
\left<\nabla_{\boldsymbol{\theta}} Q \left(\phi(s,a^1_{\boldsymbol{\theta}^t},a^2_{\boldsymbol{\theta}^t_*});\boldsymbol{\theta}^0\right), \boldsymbol{\theta}^t-\boldsymbol{\theta}^t_*\right>&\leq& \widehat{Q}^{\#}(s;\boldsymbol{\theta}^t)-\widehat{Q}^{\#}(s;\boldsymbol{\theta}^t_*) \\
&\leq& 
\left<\nabla_{\boldsymbol{\theta}} Q \left(\phi(s,a^1_{\boldsymbol{\theta}^t_*},a^2_{\boldsymbol{\theta}^t});\boldsymbol{\theta}^0\right), \boldsymbol{\theta}^t-\boldsymbol{\theta}^t_*\right>.
\end{eqnarray*}
Therefore, 
\begin{eqnarray}
\label{eq:minimax-est}
\left|\widehat{Q}^{\#}(s;\boldsymbol{\theta}^t)-\widehat{Q}^{\#}(s;\boldsymbol{\theta}^t_*) \right|&\leq& \max \left\{ \Big|\left<\nabla_{\boldsymbol{\theta}} Q \left(\phi(s,a^1_{\boldsymbol{\theta}^t},a^2_{\boldsymbol{\theta}^t_*});\boldsymbol{\theta}^0\right), \boldsymbol{\theta}^t-\boldsymbol{\theta}^t_*\right>\Big|,\right. \nonumber \\
&\ &\left.\Big|\left<\nabla_{\boldsymbol{\theta}} Q \left(\phi(s,a^1_{\boldsymbol{\theta}^t_*},a^2_{\boldsymbol{\theta}^t});\boldsymbol{\theta}^0\right), \boldsymbol{\theta}^t-\boldsymbol{\theta}^t_*\right> \Big| \right\}.
\end{eqnarray}
By Assumption \ref{as:minimax-regularity}, we compute 
\begin{eqnarray}
\label{eq:minimax-contraction}
\mathbb{E}_{\mu,\pi,\mathbb{P}}\left[\left|\widehat{Q}^{\#}(s;\boldsymbol{\theta}^t)-\widehat{Q}^{\#}(s;\boldsymbol{\theta}^t_*) \right|^2 \right] &\overset{(i)}{\leq}& \left(\boldsymbol{\theta}^t-\boldsymbol{\theta}^t_*\right)^\top \Sigma_\pi^*(\boldsymbol{\theta}^t, \boldsymbol{\theta}^t_*) \left(\boldsymbol{\theta}^t-\boldsymbol{\theta}^t_*\right) \nonumber \\
&\overset{(ii)}{\leq}& \frac{(1-\nu)^2}{\gamma^2} \cdot \left(\boldsymbol{\theta}^t-\boldsymbol{\theta}^t_*\right)^\top \Sigma_\pi \left(\boldsymbol{\theta}^t-\boldsymbol{\theta}^t_*\right) \\
&=& \frac{(1-\nu)^2}{\gamma^2} \cdot \mathbb{E}_{\mu,\pi,\mathbb{P}}\left[\left|\widehat{Q}(s,a^1,a^2;\boldsymbol{\theta}^t)-\widehat{Q}(s,a^1,a^2;\boldsymbol{\theta}^t_*) \right|^2 \right], \nonumber
\end{eqnarray}
where (i) is due to \eqref{eq:minimax-est}, and (ii) is due to Assumption \ref{as:minimax-regularity}.
Similar to Lemma \ref{lemma:I4-ql}, we can also obtain \eqref{eq:lemma-I4-decompose-ql} in this lemma. By substituting \eqref{eq:minimax-contraction} into \eqref{eq:lemma-I4-decompose-ql}, the proof can be completed.

\end{proof}

Now we are ready to prove Theorem \ref{theorem:minimax-pi-f}. For a little notation abuse, we redefine
\begin{eqnarray*}
\mathbf{g}(\boldsymbol{\theta}^t) &=& \Delta(s_t,a^1_t,a^2_t,s'_t,\boldsymbol{\theta}^t)\cdot  \nabla_{\boldsymbol{\theta}}  Q(\boldsymbol{x}_t;\boldsymbol{\theta}^t), \quad
\Bar{\mathbf{g}}(\boldsymbol{\theta}^t) \ = \ \mathbb{E}_{\mu, \pi,\mathbb{P}}\left[\mathbf{g}(\boldsymbol{\theta}^t)\right]\\ 
\mathbf{m}(\boldsymbol{\theta}^t) &=& \widehat{\Delta} (s_t,a^1_t,a^2_t,s'_t,\boldsymbol{\theta}^t)\cdot\nabla_{\boldsymbol{\theta}}  Q(\boldsymbol{x}_t;\boldsymbol{\theta}^0),
\quad \Bar{\mathbf{m}}(\boldsymbol{\theta}^t)\ =\ \mathbb{E}_{\mu, \pi,\mathbb{P}}\left[\mathbf{m}(\boldsymbol{\theta}^t)\right],
\end{eqnarray*}
where 
\begin{eqnarray*}
\Delta(s_t,a^1_t,a^2_t,s'_t,\boldsymbol{\theta}^t)&=& Q(\boldsymbol{x}_t;\boldsymbol{\theta}^t)-\left(r(s_t, a_t)+\gamma \max_{b^1\in\mathcal{A}} \min_{b^2\in\mathcal{A}} Q(\phi(s_{t+1},b^1,b^2);\boldsymbol{\theta}^t)\right), \\
\widehat{\Delta} (s_t,a^1_t,a^2_t,s'_t,\boldsymbol{\theta}^t) &=&\widehat{Q}(\boldsymbol{x}_t;\boldsymbol{\theta}^t)-\left(r(s_t, a_t)+\gamma \max_{b^1\in\mathcal{A}} \min_{b^2\in\mathcal{A}} \widehat{Q}(\phi(s_{t+1},b^1,b^2);\boldsymbol{\theta}^t)\right).
\end{eqnarray*}
Let $\Delta_t=\Delta(s_t,a^1_t,a^2_t,s'_t,\boldsymbol{\theta}^t)$ and $\widehat{\Delta}_t=\widehat{\Delta} (s_t,a^1_t,a^2_t,s'_t,\boldsymbol{\theta}^t)$. After redefining the corresponding notation, we can similarly derive \eqref{eq:err-decompose} and adopt the associated lemmas. Due to the introduction of the additional Assumption \ref{as:minimax-regularity}, we provide Lemma \ref{lemma:I4-minimax-ql}, ensuring that terms $\mathbf{I}_1 \sim \mathbf{I}_4$ can be estimated. The remainder of the proof is entirely analogous to Sections \ref{section:appendix-pi-f} and \ref{sect:proof-ql}. Thus, we conclude the proof.

\section{Supporting Lemmas for Multi-layer Neural Network}
\label{appendix:nn}
Recalling the definition of the parameterized Q-function,
we present the following lemmas related to neural network functions, which play a crucial role in illustrating the main results of our paper. $\left\{\tau_i>0\right\}_{i=1, \ldots, 10}$ mentioned below are universal constants. 

\begin{lemma}
\label{lemma:theta-bound}
For any $t\in \{1,2,\cdots,T\}$, we have
$$
\|\boldsymbol{\theta}^t\|\leq \tau_1 \sqrt{m}, \quad w.p. 1-L\exp{(-\tau_2 m)}.
$$
\end{lemma}

\begin{proof}
By Lemma G.2 in \cite{du2019gradient}, $\|\boldsymbol{\theta}^0\|\leq \mathcal{O}\left(\sqrt{m}\right)$ with probability at least $1-L\exp{(-\tau_2 m)}$. Therefore, 
\begin{eqnarray*}
\|\boldsymbol{\theta}^t\|_2&\leq& \|\boldsymbol{\theta}^t-\boldsymbol{\theta}^0\|+\|\boldsymbol{\theta}^0\| \\
&\leq& \omega + \|\boldsymbol{\theta}^0\|  \\
&\leq & \mathcal{O}\left(\sqrt{m}\right).
\end{eqnarray*}
\end{proof}

\begin{lemma}
\label{lemma:x-bound}
For any $l\in \{1,2,\cdots,L\}$, we have 
$$
\|\boldsymbol{x}^{(l)}\| \leq \tau_3 \sqrt{m},\quad \text{and} \quad \left\| \nabla_{\boldsymbol{\theta}} Q(\boldsymbol{x};\boldsymbol{\theta}) \right\| \leq \tau_4, \quad w.p.\ 1-L\exp{(-\tau_2 m)}.
$$
\end{lemma}

Lemma \ref{lemma:x-bound} has been proved by \cite{tian2022performance} (Lemmas A.6$\sim$A.10). 

\begin{lemma}
\label{lemma:q-bound}
For any $t\in \{1,2,\cdots,T\}$ and $\boldsymbol{\theta}\in S_\omega$, we have
$$
\left|Q(\boldsymbol{x}_t;\boldsymbol{\theta})\right| \leq \tau_5 \sqrt{\log (T/\delta)}, \quad w.p.\ 1-\delta-L\exp{(-\tau_2 m)}.
$$
\end{lemma}

\begin{proof}
By Lemma \ref{lemma:x-bound}, we have $\frac{1}{\sqrt{m}}\|\boldsymbol{x}^{(L)}\|\leq \tau_3,\ w.p.\ 1-L\exp{(-\tau_2 m)}$. Recall the definition of the parameterized Q-function:
$$
Q(\boldsymbol{x};\boldsymbol{\theta})=\frac{1}{\sqrt{m}} \boldsymbol{b}^\top \boldsymbol{x}^{(L)},
$$
where each element of $\boldsymbol{b}$ is generated from a uniform distribution over $\{-1,+1\}$. For each $\boldsymbol{x}$, by Hoeffding inequality, we have
$$
\mathbb{P}\left(\left|\left<\frac{1}{m}\boldsymbol{x}^{(L)}, \boldsymbol{b}\right>\right|\geq t\right) \leq 2 \exp \left(-\frac{2t^2}{4\|\frac{1}{m}\boldsymbol{x}^{(L)}\|^2}\right)\overset{(i)}{\leq} 2 \exp \left(-\frac{t^2}{2\tau_3^2}\right),
$$
where (i) follows Lemma \ref{lemma:x-bound}. Substituting $\frac{\delta}{T}=2 \exp \left(-\frac{t^2}{2\tau_3^2}\right)$, we get 
$$
\mathbb{P}\left(\left|\left<\frac{1}{m}\boldsymbol{x}^{(L)}, \boldsymbol{b}\right>\right|\geq \tau_3\sqrt{2\log (T/\delta)}\right) \leq \frac{\delta}{T}.
$$
Now, by the union bound, if we set $\tau_5=\tau_3\sqrt{2}$, then
$$
\mathbb{P}\left(\max_{t\in [T]}\left|Q(\boldsymbol{x}_t;\boldsymbol{\theta})\right|\geq \tau_5 \sqrt{\log (T/\delta)}\right) \leq \sum_{t=1}^T \mathbb{P}\left(\left|Q(\boldsymbol{x}_t;\boldsymbol{\theta})\right|\geq \tau_5 \sqrt{\log (T/\delta)}\right) \leq \delta,
$$
which completes the proof.
\end{proof}

\begin{lemma}
\label{lemma:hessian}
Denote $\nabla^2_{\boldsymbol{\theta}} Q(\boldsymbol{x};\boldsymbol{\theta})$ as the Hessian matrix of $Q(\boldsymbol{x};\boldsymbol{\theta})$. Then for all $\boldsymbol{x}, \boldsymbol{\theta}\in S_\omega$, we have $w.p.\ 1-\delta$ that
$$
\| \nabla^2_{\boldsymbol{\theta}} Q(\boldsymbol{x};\boldsymbol{\theta})\|_2\leq \tau_6 m^{-\frac{1}{2}}
$$
and 
$$
\left| Q(\boldsymbol{x};\boldsymbol{\theta}) - \widehat{Q}(\boldsymbol{x};\boldsymbol{\theta})\right| \leq \tau_7 m^{-\frac{1}{2}}.
$$
\end{lemma}

\begin{proof}
The first inequality in Lemma \ref{lemma:hessian} has been proved by \cite{liu2020linearity} (Theorem 3.2), which implies that $Q(\boldsymbol{x};\boldsymbol{\theta})$ is $\mathcal{O}(m^{-\frac{1}{2}})$-smoothness w.r.t. $\theta$. Therefore, 
$$
\left| Q(\boldsymbol{x};\boldsymbol{\theta}) - \widehat{Q}(\boldsymbol{x};\boldsymbol{\theta})\right|= \left| Q(\boldsymbol{x};\boldsymbol{\theta}) - Q(\boldsymbol{x};\boldsymbol{\theta}^0)-\left<\nabla_{\boldsymbol{\theta}} Q(\boldsymbol{x};\boldsymbol{\theta}^0), \boldsymbol{\theta}^0-\boldsymbol{\theta}\right>\right|=\mathcal{O}(m^{-\frac{1}{2}}).
$$
\end{proof} 

\begin{lemma}
\label{lemma:I1-and-I2}
Let $\boldsymbol{\theta} \in S_\omega$ with the radius satisfying $\omega=\mathcal{O}(1)$.
Then for all $\|\boldsymbol{x}\|_2=1$ and $\boldsymbol{\theta}^t_* \in S_{2\omega}$ in the neural temporal difference learning algorithm \ref{alg:Neural-TD}, it holds that
\begin{eqnarray*}
\left|\left<\mathbf{g}\left(\boldsymbol{\theta}^t\right)-\mathbf{m}\left(\boldsymbol{\theta}^t\right), \boldsymbol{\theta}^t-\boldsymbol{\theta}^{t+1}_*\right>\right| &\leq & \left(\tau_8\tau_9 \omega m^{-\frac{1}{2}}\sqrt{ \log (T / \delta)}+2\tau_4 \tau_8 m^{-\frac{1}{2}}\right)\left\| \boldsymbol{\theta}^t-\boldsymbol{\theta}^{t+1}_*\right\|
\end{eqnarray*}
with probability at least $1-2\delta-2L\exp{(-\tau_2 m)}$ over the randomness of the initial point, and $\left\|\mathbf{g}\left(\boldsymbol{\theta}^t\right)\right\|_2 \leq \tau_{10} \sqrt{\log (T / \delta)}$ holds with probability at least $1-\delta-2L\exp{(-\tau_2 m)}$. 
\end{lemma}

\begin{proof}
Note that 
\begin{eqnarray}
\label{eq:nn-grad-gap}
\left\|\mathbf{g}\left(\boldsymbol{\theta}^t\right)-\mathbf{m}\left(\boldsymbol{\theta}^t\right) \right\| &=& \left\| \Delta(\boldsymbol{x}_t, \boldsymbol{x}_{t+1},\boldsymbol{\theta}^t)\cdot  \nabla_{\boldsymbol{\theta}}  Q(\boldsymbol{x}_t;\boldsymbol{\theta}^t)-\widehat{\Delta} (\boldsymbol{x}_t, \boldsymbol{x}_{t+1},\boldsymbol{\theta}^t)\cdot\nabla_{\boldsymbol{\theta}}  Q(\boldsymbol{x}_t;\boldsymbol{\theta}^0)\right\| \nonumber \\
&\leq & \left\| \Delta(\boldsymbol{x}_t, \boldsymbol{x}_{t+1},\boldsymbol{\theta}^t)\cdot \left( \nabla_{\boldsymbol{\theta}}  Q(\boldsymbol{x}_t;\boldsymbol{\theta}^t)- \nabla_{\boldsymbol{\theta}}  Q(\boldsymbol{x}_t;\boldsymbol{\theta}^0)\right) \right\| \\
&\ & + \left\| \left(\Delta(\boldsymbol{x}_t, \boldsymbol{x}_{t+1},\boldsymbol{\theta}^t)-\widehat{\Delta} (\boldsymbol{x}_t, \boldsymbol{x}_{t+1},\boldsymbol{\theta}^t)\right)\cdot \nabla_{\boldsymbol{\theta}}  Q(\boldsymbol{x}_t;\boldsymbol{\theta}^0) \right\|,\nonumber
\end{eqnarray}
where 
\begin{eqnarray*}
\Delta(\boldsymbol{x}_t, \boldsymbol{x}_{t+1},\boldsymbol{\theta}^t)&=& Q(\boldsymbol{x}_t;\boldsymbol{\theta}^t)-\left(r(s_t, a_t)+\gamma Q(\boldsymbol{x}_{t+1};\boldsymbol{\theta}^t)\right), \\
\widehat{\Delta} (\boldsymbol{x}_t, \boldsymbol{x}_{t+1},\boldsymbol{\theta}^t) &=&\widehat{Q}(\boldsymbol{x}_t;\boldsymbol{\theta}^t)-\left(r(s_t, a_t)+\gamma  \widehat{Q}(\boldsymbol{x}_{t+1};\boldsymbol{\theta}^t)\right).
\end{eqnarray*}
For the first term in \eqref{eq:nn-grad-gap}, by Lemma \ref{lemma:q-bound}, we have $w.p.\ 1-\delta-L\exp{(-\tau_2 m)}$ that
$$
\left|\Delta(\boldsymbol{x}_t, \boldsymbol{x}_{t+1},\boldsymbol{\theta}^t)\right|\leq \tau_{8}\sqrt{ \log (T / \delta)}.
$$
By Lemma \ref{lemma:hessian}, we get that
$$
\left\| \nabla_{\boldsymbol{\theta}}  Q(\boldsymbol{x}_t;\boldsymbol{\theta}^t)- \nabla_{\boldsymbol{\theta}}  Q(\boldsymbol{x}_t;\boldsymbol{\theta}^0)\right\|\leq \| \nabla^2_{\boldsymbol{\theta}} Q(\boldsymbol{x_t};\boldsymbol{\theta_t})\|_2 \cdot \|\boldsymbol{\theta}^t-\boldsymbol{\theta}^0\|\leq \tau_9 \omega m^{-\frac{1}{2}}.
$$
Therefore,
\begin{equation}
\label{eq:nn-grad-gap-1}
\left\| \Delta(\boldsymbol{x}_t, \boldsymbol{x}_{t+1},\boldsymbol{\theta}^t)\cdot \left( \nabla_{\boldsymbol{\theta}}  Q(\boldsymbol{x}_t;\boldsymbol{\theta}^t)- \nabla_{\boldsymbol{\theta}}  Q(\boldsymbol{x}_t;\boldsymbol{\theta}^0)\right) \right\|\leq \tau_8\tau_9 \omega m^{-\frac{1}{2}}\sqrt{ \log (T / \delta)}. 
\end{equation}

For the second term in \eqref{eq:nn-grad-gap}, we decompose it into
\begin{eqnarray}
\label{eq:nn-grad-gap-2}
&\ & \left\| \left(\Delta(\boldsymbol{x}_t, \boldsymbol{x}_{t+1},\boldsymbol{\theta}^t)-\widehat{\Delta} (\boldsymbol{x}_t, \boldsymbol{x}_{t+1},\boldsymbol{\theta}^t)\right)\cdot \nabla_{\boldsymbol{\theta}}  Q(\boldsymbol{x}_t;\boldsymbol{\theta}^0) \right\| \nonumber\\
&\leq & \left\| \left(Q(\boldsymbol{x},\boldsymbol{\theta}^t)-\widehat{Q}(\boldsymbol{x},\boldsymbol{\theta}^t)\right)\cdot \nabla_{\boldsymbol{\theta}}  Q(\boldsymbol{x}_t;\boldsymbol{\theta}^0) \right\|+\left\| \left(Q(\boldsymbol{x}_{t+1};\boldsymbol{\theta}^t)-\widehat{Q}(\boldsymbol{x}_{t+1};\boldsymbol{\theta}^t)\right)\cdot \nabla_{\boldsymbol{\theta}}  Q(\boldsymbol{x}_t;\boldsymbol{\theta}^0) \right\|\nonumber \\
&\leq & \left | Q(\boldsymbol{x},\boldsymbol{\theta}^t)-\widehat{Q}(\boldsymbol{x},\boldsymbol{\theta}^t)\right|\cdot\left\| \nabla_{\boldsymbol{\theta}}  Q(\boldsymbol{x}_t;\boldsymbol{\theta}^0) \right\|+\left |Q(\boldsymbol{x}_{t+1};\boldsymbol{\theta}^t)-\widehat{Q}(\boldsymbol{x}_{t+1};\boldsymbol{\theta}^t)\right|\cdot\left\| \nabla_{\boldsymbol{\theta}}  Q(\boldsymbol{x}_t;\boldsymbol{\theta}^0) \right\|\nonumber \\
&\overset{(i)}{\leq}& 2\tau_4 \tau_8 m^{-\frac{1}{2}},
\end{eqnarray}
with probability at least $w.p.\ 1-\delta-L\exp{(-\tau_2 m)}$, where (i) is due to Lemmas \ref{lemma:q-bound} and \ref{lemma:hessian}. Plugging \eqref{eq:nn-grad-gap-1} and \eqref{eq:nn-grad-gap-2} into \eqref{eq:nn-grad-gap} yields that
\begin{eqnarray*}
\left|\left<\mathbf{g}\left(\boldsymbol{\theta}^t\right)-\mathbf{m}\left(\boldsymbol{\theta}^t\right), \boldsymbol{\theta}^t-\boldsymbol{\theta}^{t+1}_*\right>\right| &\leq & \left\|\mathbf{g}\left(\boldsymbol{\theta}^t\right)-\mathbf{m}\left(\boldsymbol{\theta}^t\right)\right\|\cdot\left\| \boldsymbol{\theta}^t-\boldsymbol{\theta}^{t+1}_*\right\| \\
&\leq & \left(\tau_8\tau_9 \omega m^{-\frac{1}{2}}\sqrt{ \log (T / \delta)}+2\tau_4 \tau_8 m^{-\frac{1}{2}}\right)\left\| \boldsymbol{\theta}^t-\boldsymbol{\theta}^{t+1}_*\right\|.
\end{eqnarray*}
Thus we complete the proof.
\end{proof}

Lemma \ref{lemma:I1-and-I2} provides the upper bounds on $\mathbf{I}_1$ and $\mathbf{I}_2$ in Section \ref{section:appendix-pi-f}. As discussed in Section \ref{sect:base-setting}, for a finite MDP, the Gram matrix of the L-layer neural network function is positive definite and has a minimum eigenvalue of $\mathcal{O}(1)$ when the network width is sufficiently large.  This, in fact, serves as an upper bound for $m^*$ in Assumption \ref{as:lamda}. Further details are provided in Remark \ref{remark:ntk}. 

\begin{remark}
\label{remark:ntk}
In this special case, we assume that both state space and action space are finite. Let $|\mathcal{S}|$ and $|\mathcal{A}|$ represent the dimensions of the state space and action space, respectively. For simplicity of notation, we view $\mathbf{Q}(\boldsymbol{\theta})$ as an $|\mathcal{S}| |\mathcal{A}|\times 1$ column vector, with $(s,a)$ being a multi-index arranged in the lexicographical order. Let $d\sim \mu\times \pi$ and $\mathbf{D}=\text{diag}(d)$ be an $|\mathcal{S}| |\mathcal{A}|$-dimensional diagonal matrix, whose $(s,a)$-th diagonal entry is $d(s,a)$, and the order of $(s,a)$ in $\mathbf{D}$ is the same as $\mathbf{Q}(\boldsymbol{\theta})$. Denote $\mathbf{J}$ as the Jacobian matrix of $\mathbf{Q}(\boldsymbol{\theta}^0)$ and $\mathbf{J_D}=\mathbf{D}^{\frac{1}{2}}\mathbf{J}$. Thus we can rewrite $\Sigma_\pi=\mathbf{J_D^\top}\mathbf{J_D}$. Notice that $\Sigma_\pi$ is different from the Gram matrix \citet{jacot2018neural, du2018gradient, du2019gradient, cao2019generalization, allen2019convergence} in deep neural network. To derive the $\mu$-weighted Gram matrix $\text{Gram} (\boldsymbol{\theta}^0)$, we provide the following definition.

\begin{definition}
\label{def:NTK}
(\citet{du2019gradient, cao2019generalization, allen2019convergence, allen2019learning}, Neural Tangent Kernel Matrix). For any $i, j \in\left[|\mathcal{S}| |\mathcal{A}|\right]$, define
$$
\begin{aligned}
& \widetilde{\boldsymbol{\Theta}}_{i, j}^{(1)}=\boldsymbol{\Sigma}_{i, j}^{(1)}=\left\langle\widehat{\boldsymbol{x}}_i, \widehat{\boldsymbol{x}}_j\right\rangle, \quad \mathbf{A}_{i j}^{(l)}=\left(\begin{array}{cc}
\boldsymbol{\Sigma}_{i, i}^{(l)} & \boldsymbol{\Sigma}_{i, j}^{(l)} \\
\boldsymbol{\Sigma}_{i, j}^{(l)} & \boldsymbol{\Sigma}_{j, j}^{(l)}
\end{array}\right), \\
& \boldsymbol{\Sigma}_{i, j}^{(l+1)}=2 \cdot \mathbb{E}_{(u, v) \sim N\left(\mathbf{0}, \mathbf{A}_{i j}^{(l)}\right)}[\sigma(u) \sigma(v)], \\
& \widetilde{\boldsymbol{\Theta}}_{i, j}^{(l+1)}=\widetilde{\boldsymbol{\Theta}}_{i, j}^{(l)} \cdot 2 \cdot \mathbb{E}_{(u, v) \sim N\left(\mathbf{0}, \mathbf{A}_{i j}^{(l)}\right)}\left[\sigma^{\prime}(u) \sigma^{\prime}(v)\right]+\boldsymbol{\Sigma}_{i, j}^{(l+1)},
\end{aligned}
$$
where $\widehat{\boldsymbol{x}}=\sqrt{d(s,a)}\phi(s,a)$.
Then we call $\boldsymbol{\Theta}^{(L)}=\left[\left(\widetilde{\boldsymbol{\Theta}}_{i, j}^{(L)}+\boldsymbol{\Sigma}_{i, j}^{(L)}\right) / 2\right]_{|\mathcal{S}| |\mathcal{A}| \times |\mathcal{S}| |\mathcal{A}|}$ the $\mu$-weighted neural tangent kernel matrix of an $L$-layer ReLU network on $\mu$-weighted state-action pairs $\widehat{\boldsymbol{x}}_1, \ldots, \widehat{\boldsymbol{x}}_{|\mathcal{S}| |\mathcal{A}|}$.
\end{definition} 

There is a large body of work \citet{jacot2018neural, du2018gradient, du2019gradient, cao2019generalization, allen2019convergence} exploring the positive definiteness of $\text{Gram} (\boldsymbol{\theta}^0)$ in the literature. Suppose that for all pairs $(s,a),(s',a')$, $\|\phi(s,a)\|=1, d(s,a)>0$ and $\phi(s,a)\nparallel \phi(s',a')$. The results of Theorem 1 and Proposition 2 in \citet{jacot2018neural} shows that for an $L$-layer neural network with Gaussian initialization parameters, 

$$
\left< \nabla_{\boldsymbol{\theta}} Q(\widehat{\boldsymbol{x}}_i;\boldsymbol{\theta}^0), \nabla_{\boldsymbol{\theta}} Q(\widehat{\boldsymbol{x}}_j,\boldsymbol{\theta}^0)  \right> \rightarrow \boldsymbol{\Theta}^{(L)}_{i,j}, \quad \text{as}\ m\rightarrow \infty.
$$
That is, under the NTK regime, the $\mu$-weighted Gram matrix $\text{Gram} (\boldsymbol{\theta}^0)=\mathbf{J_D} \mathbf{J_D^\top}$ converges to $\boldsymbol{\Theta}^{(L)}$ when $m$ is sufficiently large.
Let $\lambda_{\min}\left(\boldsymbol{\Theta}^{(L)}\right)=2\lambda'=\mathcal{O}(1)$.
For any $\delta \in(0,1)$, there exists $m^*=\text{Poly}(|\mathcal{S}| |\mathcal{A}|, L, \delta, \lambda')$ such that if $m\geq m^*$, we have
$$
\lambda_{\min}\left(\text{Gram} (\boldsymbol{\theta}^0)\right)\geq \lambda'\quad w.p.\ 1-\delta.
$$
This signifies that if the network width $m\geq m^*$, then we have $\lambda_0\geq \lambda_{\min}\left(\text{Gram} (\boldsymbol{\theta}^0)\right)\geq \lambda'$, thereby substantiating our claim.
\end{remark}

\section{Additional Notes on the Experiments in Section \ref{sect:exp}}
In this section, we further discuss the experimental setup introduced in Section \ref{sect:exp}. As mentioned in Section \ref{sect:exp}, our experiments mainly test the following two aspects: (1) how does the network width m affects the final error of the algorithm (first two subfigures in Figure \ref{fig:td-learning}); (2) the minimum nonzero singular value in Assumption \ref{as:technical} (latter two subfigures in Figure \ref{fig:td-learning}).

For point (1), we first generate 2000 samples according to a given policy $\pi$ to imitate the Markov process of Algorithm \ref{alg:Neural-TD}. A two-layer neural network with ELU activation is introduced, and the parameters are initialized using Algorithm \ref{alg:Neural-TD}. We set the initial learning rate at 0.001 with linear decay (per epoch) and a batch size of 100. Notably, as the parameter $m$ increases, the TD algorithm demonstrates smaller final TD errors.

For point (2), our experiments are based on three main points. First, note that the norm of feature map and parameter random initialization will affect the scaling of the gradient norm w.r.t. $Q(s,a;\boldsymbol{\theta})$. Thus we employ the ratio $r=\sigma_{\max}/\sigma_{\min}$ to characterize the minimum non-zero singular value in order to eliminate the impact of numerical scaling. Second, we set varying network widths to verify the existence of $\lambda_0$. Finally, it's tough to directly obtain the joint distribution of $(s,a)$ with a fixed learning policy. However, we have that $\left\|\frac{1}{N}\sum_{(s,a)\sim\mu\times\pi}\nabla_{\boldsymbol{\theta}}Q(s,a;\boldsymbol{\theta})\nabla_{\boldsymbol{\theta}}Q(s,a;\boldsymbol{\theta})^\top-\Sigma_\pi\right\|_2=\mathcal{O}(1/N)$. To avoid the effects of sampling randomness, we estimate $\Sigma_\pi$ from different samples. The experiments demonstrate that the minimum non-zero singular value $\lambda_0$ converges to a constant as $m$ increases.


\end{document}